\documentclass[10pt,final,journal,twocolumn]{IEEEtran}

\usepackage{amsmath,amssymb,color,amsmath, graphicx, float, subcaption, mathrsfs,color} \usepackage[ruled]{algorithm2e} 
\usepackage{bm}
\usepackage{epstopdf }
\usepackage[font={small}]{caption}
\usepackage[font={footnotesize}]{subcaption}
\usepackage{cite}
\usepackage{amsthm}
\usepackage{url}

\renewcommand{\Bbb}{\mathbb}

\theoremstyle{plain}
\newtheorem{theorem}{\textbf{Theorem}}
\newtheorem{definition}{\textbf{Definition}}

\newtheorem{proposition}[theorem]{\textbf{Proposition}}

\newtheorem{observation}{\textbf{Observation}}

\def\A{\bm{A}}

\def\D{\bm{D}}

\def\I{\bm{I}}

\def\P{\bm{P}}

\def\U{\bm{U}}
\def\V{\bm{V}}

\def\e{\bm{e}}

\def\v{\bm{v}}

\def\x{\bm{x}}
\def\y{\bm{y}}
\def\z{\bm{z}}

\def\bGamma{\bm{\Gamma}}

\def\bLambda{\boldsymbol{\Lambda}}

\def\bOmega{\boldsymbol{\Omega}}

\def\argmin#1{\underset{#1}{\textrm{argmin}}}

\def\minim#1{\underset{#1}{\textrm{min}}}
\def\Exp{\mathbb{E}}

\def\0{\mathbf{0}}
\def\1{\mathbf{1}}

\newcommand{\tomt}[1]{\textcolor{black}{#1}}
\newcommand{\tomtb}[1]{\textcolor{black}{#1}}

\SetKwInput{kwEvaluate}{Evaluate}
\SetKwInput{kwSort}{Sort}
\SetKwInput{kwInput}{Input}
\SetKwInput{kwOutput}{Output}
\SetKwInput{kwInitialize}{Initialize}
\SetKwInput{kwParameters}{Params.}
\SetKwInput{kwDefaults}{Default init.}

\begin{document}

\title{
Back-Projection based Fidelity Term for Ill-Posed Linear Inverse Problems}
\author{
    Tom~Tirer
    and Raja~Giryes \\

    \thanks{This work was supported by the European research council (ERC starting grant 757497 PI Giryes).
The authors are with the School of Electrical Engineering, Tel Aviv University, Tel Aviv 69978, Israel. (email: tirer.tom@gmail.com, raja@tauex.tau.ac.il)}
          }

\maketitle

\begin{abstract}

Ill-posed linear inverse problems appear in many image processing applications, such as deblurring, super-resolution and compressed sensing. Many restoration strategies involve minimizing a cost function, which is composed of fidelity and prior terms, balanced by a regularization parameter. While a vast amount of research has been focused on different prior models, the fidelity term is almost always chosen to be the least squares (LS) objective, that encourages fitting the linearly transformed optimization variable to the observations. In this paper, we examine a different fidelity term, which has been implicitly used by the recently proposed iterative denoising and backward projections (IDBP) framework. This term encourages agreement between the projection of the optimization variable onto the row space of the linear operator and the pseudo-inverse of the linear operator ("back-projection") applied on the observations. We analytically examine the difference between the two fidelity terms for Tikhonov regularization and identify cases (such as a badly conditioned linear operator) where the new term has an advantage over the standard LS one. 
Moreover, we demonstrate empirically that the behavior of the two induced cost functions for sophisticated convex and non-convex priors, such as total-variation, BM3D, and deep generative models, 
correlates with the obtained theoretical analysis.

\end{abstract}


\section{Introduction}
\label{sec_int}

Inverse problems appear in many fields of science and engineering, where the goal is to recover a signal from its observations that are obtained by some acquisition process. In image processing, the observations are usually a degraded version of the latent image, which may be noisy, blurred, downsampled, or all together. Such observation models, and others, can be formulated by a linear model
\begin{align}
\label{Eq_general_model}
\y = \A\x + \e,
\end{align}
where $\x \in \Bbb R^n$ represents the unknown original image, $\y \in \Bbb R^m$ represents the observations, $\A$ is an $m \times n$ degradation matrix (sometimes also referred to as the measurement matrix) and $\e \in \Bbb R^m$ is a noise vector.
For example, this model corresponds to the problem of denoising \cite{rudin1992nonlinear, buades2005review, elad2006image, dabov2007image} when $\A$ is the $n \times n$ identity matrix $\I_n$; inpainting \cite{bertalmio2000image, criminisi2004region, elad2005simultaneous} when $\A$ is an $m \times n$ sampling matrix (i.e. a selection of m rows of $\I_n$); deblurring \cite{guerrero2008image, danielyan2012bm3d} when $\A$ is a blur operator; super-resolution \cite{yang2010image, dong2013nonlocally} if $\A$ is a composite operator of blurring (e.g. anti-aliasing filtering) and down-sampling; and compressed sensing when $\A$ is a (random) measurement matrix ($m \ll n$) and the signal is sparse under some basis representation \cite{donoho2006compressed, duarte2008single, candes2008introduction} or resides in a general union of low-dimensional subspaces \cite{blumensath2011sampling, tirer2018generalizing}. 

The inverse problems represented by \eqref{Eq_general_model} are usually ill-posed, i.e. 
the measurements do not suffice for obtaining a successful reconstruction. 
Therefore, a vast amount of research has focused on designing good prior models for natural images. In fact, many of the methods for the problems mentioned above differ only in their prior assumptions and not in the way that they enforce fidelity to the observations.

To be more formal, a common strategy for recovering $\x$ aims at minimizing a cost function of the form
\begin{align}
\label{Eq_cost_func_general}
f(\tilde{\x}) = \ell(\tilde{\x}) + \beta s(\tilde{\x}),
\end{align}
where $\ell(\tilde{\x})$ is a fidelity term, $s(\tilde{\x})$ is a prior term (can be also referred to as the regularizer), $\beta$ is a positive scalar that controls the level of regularization, and $\tilde{\x}$ is the optimization variable. 
Many different prior functions are used in the literature, whether explicitly, e.g.  total-variation (TV) \cite{rudin1992nonlinear}, or implicitly, e.g. BM3D \cite{dabov2007image} and deep generative models \cite{bora2017compressed}. Yet, most of the works use a typical least squares (LS) fidelity term 
\begin{align}
\label{Eq_fidelity_typical}
\ell_{LS}(\tilde{\x}) \triangleq  \frac{1}{2} \| \y-\A\tilde{\x} \|_2^2,
\end{align}
where $\| \cdot \|_2$ stands for the Euclidean norm. 
The frequent usage of this term is probably also motivated by the fact that it can be 
derived from the negative log-likelihood function, under the assumption that the noise $\e$ is a vector of i.i.d. Gaussian random variables $e_i \sim \mathcal{N}(0,\sigma_e^2)$. 
However, note that, in general, maximum likelihood estimation has optimality properties only when the number of measurements is {\em much larger} than the number of unknown variables, which is obviously {\em not the case} in ill-posed problems.

In this paper, we examine a different fidelity term, which has been implicitly used by the recently proposed iterative denoising and backward projections (IDBP) framework \cite{tirer2019image} (we elaborate on 
\tomt{this method in the appendix}).
Under the practical assumptions that $m \leq n$ 
and $\mathrm{rank}(\A)=m$, 
we examine the fidelity term
\begin{align}
\label{Eq_fidelity_idbp}
\ell_{BP}(\tilde{\x}) \triangleq  \frac{1}{2} \| \A^\dagger \y - \A^\dagger \A \tilde{\x} \|_2^2,
\end{align}
where $\A^\dagger \triangleq \A^T(\A\A^T)^{-1}$ is the pseudoinverse of the full row-rank matrix $\A$.
Note that $\P_A \triangleq \A^\dagger \A$ is an orthogonal projection onto the row space of $\A$\footnote{\tomt{In row space of $\A$, we mean the subspace spanned by the rows of $\A$.}}, and that $\A^\dagger$ can be interpreted as a ''back-projection'' (BP) from $\A \Bbb R^n$ 
back to $\Bbb R^n$. Therefore, 
the fidelity \eqref{Eq_fidelity_idbp} encourages agreement between $\P_A \tilde{\x}$ ---the projection of 
$\tilde{\x}$ onto the row space of $\A$, and $\A^\dagger \y$ ---the back-projection of the measurements. 
In general, this is different than $\ell_{LS}(\tilde{\x})$  that encourages agreement between $\A\tilde{\x}$ and $\y$.
Note that in the noiseless case, i.e. when $\y=\A\x$, the terms in \eqref{Eq_fidelity_typical} and \eqref{Eq_fidelity_idbp} are  translated to fitting $\A \tilde{\x}$ to $\A \x$ and $\P_A \tilde{\x}$ to $\P_A \x$, respectively.

Note that for some inverse problems $\ell_{LS}(\tilde{\x})$ and $\ell_{BP}(\tilde{\x})$ may coincide. 
For example, in image inpainting, where $\A$ is a selection of $m$ rows of $\I_n$, we have that $\A^\dagger = \A^T$ is an $n \times m$ matrix that merely pads with $n-m$ zeros the vector on which it is applied, 
and so $\| \A^\dagger ( \y - \A\tilde{\x}) \|_2^2 = \| \y-\A\tilde{\x} \|_2^2$.
Therefore, we specifically focus on three popular inverse problems: super-resolution, deblurring and certain compressed sensing scenarios, where the two fidelity terms, $\ell_{LS}(\tilde{\x})$ and $\ell_{BP}(\tilde{\x})$, are indeed very different.

{\bf Contribution.} 
This work makes a first attempt towards characterizing for which observation model $\A$ and prior $s(\tilde{\x})$ it is better to use each of the following objectives:
\begin{align}
\label{Eq_cost_typical}
f_{LS}(\tilde{\x}) &\triangleq \frac{1}{2} \| \y-\A\tilde{\x} \|_2^2 + \beta s(\tilde{\x}), \\
\label{Eq_cost_bp}
f_{BP}(\tilde{\x}) &\triangleq \frac{1}{2} \| \A^\dagger \y - \A^\dagger \A \tilde{\x} \|_2^2 + \beta s(\tilde{\x}).
\end{align}
Particularly, for $s(\tilde{\x})$ being the Tikhonov regularization (the $\ell_2$ prior), where 
closed-form solutions exist, we derive analytical expressions for the estimations' mean square error (MSE) that allow to examine which fidelity term is preferable. For example, we show that in the noiseless case $f_{BP}(\tilde{\x})$ yields provably better restoration than $f_{LS}(\tilde{\x})$ 
\tomtb{
if the condition number of $\A\A^T$ (i.e. the ratio between the largest and smallest squared singular values of $\A$) is large, e.g. in typical super-resolution problems.
}

For sophisticated convex and non-convex priors, such as TV \cite{rudin1992nonlinear}, BM3D \cite{dabov2007image}, and DCGAN \cite{radford2015unsupervised}, analytical analysis is intractable. Therefore, we perform an intensive empirical study, where we use the same optimization method (FISTA \cite{beck2009fast} or ADAM \cite{kingma2014adam}) to minimize each of the two different cost functions. 
Interestingly, we demonstrate that the behavior for the sophisticated priors strongly correlates with properties for which we establish concrete mathematical reasoning in the case of $\ell_2$ priors.

\tomt{
Another contribution of the paper that is deferred to the appendix is showing that IDBP framework \cite{tirer2019image}, which has achieved excellent results for deblurring \cite{tirer2019image, tirer2018icip} and super-resolution \cite{tirer2018super} 
 is in fact the proximal gradient method \cite{beck2009fast, combettes2011proximal} (popularized under the name ISTA) applied on $f_{BP}(\tilde{\x})$. This derivation of IDBP is completely different, and arguably simpler, than the way it is developed in \cite{tirer2019image}.
}

The paper is organized as follows. 
Section \ref{sec_math} includes mathematical analysis of the two cost functions for the case of $\ell_2$-type priors. The analytical results are verified in Section \ref{sec_math_ver}.
In Section \ref{sec_exp} the two cost functions are empirically examined 
for different sophisticated priors. 
Section \ref{sec_conclusion} concludes the paper.

\section{Mathematical Analysis for $\ell_2$ Priors}
\label{sec_math}

In this section, we analyze the performance of the new cost function \eqref{Eq_cost_bp} and compare it to \eqref{Eq_cost_typical} 
for 
a type of $\ell_2$ prior functions, 
for which the closed-form solutions of \eqref{Eq_cost_typical} and \eqref{Eq_cost_bp} lead to a tractable performance analysis.
We start with specifying the required assumptions, then we derive the estimators and expressions for their expected mean square error. Finally, the error expressions are compared and several observations are stated.

\subsection{Assumptions}
\label{sec_math_assump}

In order to allow a concrete mathematical comparison between $f_{BP}(\tilde{\x})$ and $f_{LS}(\tilde{\x})$, in the theoretical analysis we restrict our discussion to $\ell_2$ prior functions of the form $s(\tilde{\x})=\tomt{\frac{1}{2}}\|\D\tilde{\x}\|_2^2=\tomt{\frac{1}{2}}\tilde{\x}^T\D^T\D\tilde{\x}$, where $\D^T\D$ is a positive-definite matrix. This prior is often referred to as Tikhonov regularization and is one of the most widely used methods to solve ill-posed inverse problems.
Yet, for obtaining analytical results, we further focus on a more specific type of this prior---we require that both $\A$ and $\D$ have 
the same right singular vectors. 
Let us define the singular value decomposition (SVD) of the $m \times n$ matrix $\A=\U \bLambda \V^T$, where $\U$ is an $m \times m$ orthogonal matrix whose columns are the left singular vectors, $\bLambda$ is an $m \times n$ {\em rectangular} diagonal matrix with nonzero singular values $\{\lambda_i \}_{i=1}^{m}$ on the diagonal, and $\V$ is an $n \times n$ orthogonal matrix whose columns are the right singular vectors. 
The property that $\{\lambda_i \}_{i=1}^{m}$ are strictly positive 
follows from our assumptions in Section \ref{sec_int}, that $m \leq n$ and $\mathrm{rank}(\A)=m$. 
For $\D$, essentially, we assume that $\D^T\D=\V \bGamma^2 \V^T \succ 0$, where $\bGamma^2$ is an $n \times n$ diagonal matrix of nonzero eigenvalues $\{\gamma_i^2 \}_{i=1}^{n}$.

The assumption above is required because, as far as we know, currently there is no known analytical expression for the eigen-decomposition of arbitrary matrices $\A^T\A + \D^T\D$ which is required for our analysis \cite{horn2012matrix}. 
Yet, this assumption holds in some practical cases, e.g. if $\A$ and $\D$ are circulant matrices (and thus diagonalized by the discrete Fourier transform), or if $\D=\I_n$ (i.e. least-norm regularization).

\subsection{Performance analysis}
\label{sec_math_perf}

Let us start with obtaining closed-form expressions for the estimators $\hat{\x}_{LS}$ and $\hat{\x}_{BP}$, which minimize $f_{LS}(\tilde{\x})$ and $f_{BP}(\tilde{\x})$, respectively. Due to the convexity of the cost functions, this is done simply by equating their gradients to zero
\begin{align}
\label{Eq_cost_typical_est}
\nabla f_{LS}(\tilde{\x}) &= - \A^T ( \y - \A \tilde{\x} ) + \beta \D^T\D \tilde{\x} = \0  \nonumber \\
& \Rightarrow \hat{\x}_{LS}=(\A^T\A + \beta \D^T\D)^{-1}\A^T \y,
\end{align}
\begin{align}
\label{Eq_cost_idbp_est}
\nabla f_{BP}(\tilde{\x}) &= - \P_A ( \A^\dagger \y - \P_A \tilde{\x} ) + \beta \D^T\D \tilde{\x} = \0  \nonumber \\
& \Rightarrow \hat{\x}_{BP}=(\P_A + \beta \D^T\D)^{-1} \A^\dagger \y.
\end{align}
\tomt{
In \eqref{Eq_cost_idbp_est} we use the properties $\P_A \triangleq \A^\dagger \A = \P_A^T=\P_A^2$ and $\P_A \A^\dagger = \A^\dagger$. 
}
We turn to compute the expected mean square errors (MSEs) of the estimators, conditioned on $\x$, under the assumptions that $\Exp[\e]=\0$ and $\Exp[\e\e^T]=\sigma_e^2\I_m$.
To ease formulations, we define the $n-m$ zero eigenvalues of $\A^T\A$ (i.e. zeros in the diagonal of $\bLambda^T\bLambda$) by $\{\lambda_i^2 \}_{i=m+1}^{n}$.

The computation of the MSE of $\hat{\x}_{LS}$ is given by
\begin{align}
\label{Eq_cost_typical_mse}
& MSE_{LS} = \Exp\|\hat{\x}_{LS} - \x \|_2^2 \nonumber \\
& \hspace{0pt} = \Exp \left \|(\A^T\A + \beta \D^T\D)^{-1}\A^T (\A\x+\e) - \x \right \|_2^2  \nonumber \\
& \hspace{0pt} = \tomt{ \left \|\left ( (\A^T\A + \beta \D^T\D)^{-1}\A^T\A - \I_n \right )\x \right \|_2^2 } \nonumber\\
&\hspace{10pt}  \tomt{ + 2 \Exp \left [ \e \right ]^T \A (\A^T\A + \beta \D^T\D)^{-2} \A^T \A \x }  \nonumber\\
&\hspace{10pt}  \tomt{ - 2 \Exp \left [ \e \right ]^T \A (\A^T\A + \beta \D^T\D)^{-1} \x }  \nonumber\\
&\hspace{10pt}  \tomt{ + \Exp \left [ \e^T \A (\A^T\A + \beta \D^T\D)^{-2} \A^T \e \right ] }  \nonumber\\
& \hspace{0pt} = \tomt{ \left \|\left ( (\A^T\A + \beta \D^T\D)^{-1}\A^T\A - \I_n \right )\x \right \|_2^2} \nonumber\\
&\hspace{10pt}  \tomt{ + \mathrm{Tr} \left ( (\A^T\A + \beta \D^T\D)^{-2} \A^T \Exp \left [ \e \e^T \right ] \A \right ) } \nonumber\\
& \hspace{0pt} = \left \|\left ( (\A^T\A + \beta \D^T\D)^{-1}\A^T\A - \I_n \right )\x \right \|_2^2 \nonumber\\
&\hspace{10pt}  + \sigma_e^2 \mathrm{Tr} \left ( (\A^T\A + \beta \D^T\D)^{-2} \A^T\A \right )  \nonumber\\
& \hspace{0pt} = \left \| \V \left ( (\bLambda^T\bLambda + \beta \bGamma^2)^{-1}\bLambda^T\bLambda - \I_n \right ) \V^T \x \right \|_2^2 \nonumber\\
&\hspace{10pt}  + \sigma_e^2 \mathrm{Tr} \left ( \V (\bLambda^T\bLambda + \beta \bGamma^2)^{-2} \bLambda^T\bLambda \V^T \right ) \nonumber\\
& \hspace{0pt} = \sum \limits_{i=1}^{n} \Big ( \frac{\lambda_i^2}{\lambda_i^2+\beta \gamma_i^2} - 1 \Big )^2 [\V^T\x]_i^2 + \sigma_e^2 \sum \limits_{i=1}^{n} \frac{\lambda_i^2}{(\lambda_i^2+\beta \gamma_i^2)^2}.
\end{align}
The second equality follows from substituting \eqref{Eq_general_model} in \eqref{Eq_cost_typical_est}, \tomt{the fourth equality uses $\Exp \left [ \e \right ]=\0$ and the cyclic property of trace, the fifth equality uses $\Exp \left [ \e\e^T \right ]=\sigma_e^2\I_m$, the sixth} equality is obtained by substituting the eigen-decompositions of $\A^T\A$ and $\D^T\D$, and the last equality follows from the fact that $\V$ is an orthogonal matrix.
Therefore, 
by defining the (squared) bias and variance terms as
\begin{align}
\label{Eq_cost_typical_bias_var}
bias_{LS}^2 &\triangleq  \sum \limits_{i=1}^{m} \underbrace{ \Big ( \frac{\beta \gamma_i^2}{\lambda_i^2+\beta \gamma_i^2} \Big )^2 [\V^T\x]_i^2 }_{\triangleq bias_{LS}^{2(i)}}  +  \sum \limits_{i=m+1}^{n} [\V^T\x]_i^2 , \nonumber \\
var_{LS} &\triangleq \sum \limits_{i=1}^{m} \underbrace{ \frac{\sigma_e^2}{\lambda_i^2(1+\beta \gamma_i^2/\lambda_i^2)^2} }_{\triangleq var_{LS}^{(i)}},
\end{align}
we may write the error as 
\begin{align}
\label{Eq_cost_typical_mse2}
MSE_{LS} = bias_{LS}^2 + var_{LS}.
\end{align}
Note that the bias depends on the original image $\x$ and not on the noise, and the opposite holds for the variance. Yet, both terms are affected by the structure of $\A$.
The regularization parameters $\beta, \{\gamma_i\}$ introduce a tradeoff: increasing them reduces the variance but increases the bias.

To ease the computation of 
the MSE of $\hat{\x}_{BP}$, let us also define an indicator function $1_{i \leq m}$ that is equal to 1 if $i \leq m$ and 0 otherwise, and an $n \times n$ diagonal matrix $\I_{i \leq m}$ with $\{1_{i \leq m}\}_{i=1}^n$ on its diagonal.
The following identities are used
\begin{align}
\label{Eq_cost_idbp_identities}
\P_A &= \V \I_{i \leq m} \V^T, \nonumber \\
\A^\dagger &= \V \bLambda^T (\bLambda\bLambda^T)^{-1} \U^T,  \nonumber \\
\A^\dagger \A^{\dagger T} &= \V \bLambda^T (\bLambda\bLambda^T)^{-2} \bLambda \V^T.
\end{align}
Now, we get
\begin{align}
\label{Eq_cost_idbp_mse}
 & MSE_{BP}  = \Exp\|\hat{\x}_{BP} - \x \|_2^2 \nonumber \\
& = \Exp \left \|(\P_A + \beta \D^T\D)^{-1}\A^\dagger (\A\x+\e) - \x \right \|_2^2  \nonumber \\
& \hspace{0pt} = \tomt{ \left \|\left ( (\P_A + \beta \D^T\D)^{-1}\P_A - \I_n \right )\x \right \|_2^2 } \nonumber\\
&\hspace{10pt}  \tomt{ + 2 \Exp \left [ \e \right ]^T \A^{\dagger T} (\P_A + \beta \D^T\D)^{-2} \P_A \x }  \nonumber\\
&\hspace{10pt}  \tomt{ - 2 \Exp \left [ \e \right ]^T \A^{\dagger T} (\P_A + \beta \D^T\D)^{-1} \x }  \nonumber\\
&\hspace{10pt}  \tomt{ + \Exp \left [ \e^T  \A^{\dagger T} (\P_A + \beta \D^T\D)^{-2}  \A^{\dagger} \e \right ] }  \nonumber\\
& = \tomt{ \left \|\left ( (\P_A + \beta \D^T\D)^{-1}\P_A - \I_n \right )\x \right \|_2^2 } \nonumber\\
&\hspace{10pt}  \tomt{ +  \mathrm{Tr} \left ( (\P_A + \beta \D^T\D)^{-2} \A^\dagger \Exp \left [ \e \e^T \right ] \A^{\dagger T} \right ) } \nonumber\\
& = \left \|\left ( (\P_A + \beta \D^T\D)^{-1}\P_A - \I_n \right )\x \right \|_2^2 \nonumber\\
&\hspace{10pt}  + \sigma_e^2 \mathrm{Tr} \left ( (\P_A + \beta \D^T\D)^{-2} \A^\dagger \A^{\dagger T} \right )  \nonumber\\
&  = \left \| \V \left ( (\I_{i \leq m} + \beta \bGamma^2)^{-1}\I_{i \leq m} - \I_n \right ) \V^T \x \right \|_2^2 \nonumber\\
&\hspace{10pt}  + \sigma_e^2 \mathrm{Tr} \left ( \V (\I_{i \leq m} + \beta \bGamma^2)^{-2} \bLambda^T (\bLambda\bLambda^T)^{-2} \bLambda \V^T \right ) \nonumber\\
&  = \sum \limits_{i=1}^{n} \Big ( \frac{1_{i \leq m}}{1_{i \leq m}+\beta \gamma_i^2} - 1 \Big )^2 [\V^T\x]_i^2 + \sigma_e^2 \sum \limits_{i=1}^{n} \frac{\lambda_i^{-2}1_{i \leq m}}{(1_{i \leq m}+\beta \gamma_i^2)^2}.
\end{align}
The second equality follows from substituting \eqref{Eq_general_model} in \eqref{Eq_cost_idbp_est}, \tomt{the fourth equality uses $\Exp \left [ \e \right ]=\0$ and the cyclic property of trace, the fifth equality uses $\Exp \left [ \e\e^T \right ]=\sigma_e^2\I_m$, the sixth} equality is obtained by substituting the eigen-decompositions of $\P_A$, $\D^T\D$ and $\A^\dagger \A^{\dagger T}$, and the last equality uses the orthogonality of $\V$.
Therefore, 
by defining 
\begin{align}
\label{Eq_cost_idbp_bias_var}
bias_{BP}^2 &\triangleq  \sum \limits_{i=1}^{m} \underbrace{ \Big ( \frac{\beta \gamma_i^2}{1+\beta \gamma_i^2} \Big )^2 [\V^T\x]_i^2 }_{\triangleq bias_{BP}^{2(i)}} +  \sum \limits_{i=m+1}^{n} [\V^T\x]_i^2 , \nonumber \\
var_{BP} &\triangleq \sum \limits_{i=1}^{m} \underbrace{ \frac{\sigma_e^2}{\lambda_i^2(1+\beta \gamma_i^2)^2} }_{\triangleq var_{BP}^{(i)}},
\end{align}
we have that 
\begin{align}
\label{Eq_cost_bp_mse2}
MSE_{BP} = bias_{BP}^2 + var_{BP}.
\end{align}

Comparing  \eqref{Eq_cost_typical_bias_var} and \eqref{Eq_cost_idbp_bias_var} 
we may notice the following.
First, the term $bias_{BP}^2$ handles small $\{\lambda_i \}_{i=1}^{m}$ (i.e. singular values of $\A$ that are smaller than 1) better than $bias_{LS}^2$. However, $var_{BP}$ handles such small singular values worse than $var_{LS}$.
The opposite holds for singular values that are greater than 1.
This behavior can be formulated as the following observation.
\begin{observation}
\label{observ1}
For $\lambda_i<1$ we have that $bias_{BP}^{2(i)} < bias_{LS}^{2(i)}$ but $var_{BP}^{(i)} > var_{LS}^{(i)}$.
And, for $\lambda_i>1$ we have that $bias_{BP}^{2(i)} > bias_{LS}^{2(i)}$ but $var_{BP}^{(i)} < var_{LS}^{(i)}$.
\end{observation}
Notice that in the noiseless case $\sigma_e=0$, implying that $MSE_{LS} = bias_{LS}^2$ and $MSE_{BP} = bias_{BP}^2$.
This leads us to the following observation for the noiseless case. 

\begin{observation}
\label{observ2}
In a noiseless scenario, the relation between $\sum \limits_{i=1}^{m} bias_{BP}^{2(i)}$ and $\sum \limits_{i=1}^{m} bias_{LS}^{2(i)}$, dictates the relation between $MSE_{BP}$ and $MSE_{LS}$.
In particular, if all the singular values of $\A$ are smaller than 1, then $MSE_{BP} < MSE_{LS}$, and if all the singular values of $\A$ are greater than 1, then $MSE_{BP} > MSE_{LS}$.
\end{observation}

Note that Observation \ref{observ2} holds for any given setting of $\beta$ that is used by the two estimators. 
Therefore, these relations between $MSE_{BP}$ and $MSE_{LS}$ hold also when $\beta$ is tuned for best performance of each estimator.

\tomtb{
In practice, a different value of $\beta$ can be preferred for the different cost functions. Let us denote by $\beta_{LS}$ and $\beta_{BP}$ the regularization parameter in $\ell_{LS}(\tilde{\x})$ and $\ell_{BP}(\tilde{\x})$, respectively, and
let the singular values of $\A$ be indexed in a descending order, i.e. $\lambda_1 \geq \ldots \geq \lambda_m$.
Comparing $MSE_{BP}$ and $MSE_{LS}$ with $\beta_{BP} \neq \beta_{LS}$ leads to an additional observation for the noiseless case, which is in favor of the BP cost. 
}

\begin{observation}
\tomtb{
\label{observ3}
In a noiseless scenario, for any $\beta_{LS}$ and $\beta_{BP}=\beta_{LS}/\lambda_1^2$, we have that $MSE_{BP} \leq MSE_{LS}$. If in addition $[\V^T\x]_{{i}} \neq 0$ for some indices $2 \leq {{i}} \leq m$, then $MSE_{BP} < MSE_{LS}$ unless $\lambda_{{i}} = \lambda_1$ for all these indices.
}
\end{observation}

\begin{proof}
\tomtb{
Since $\beta_{BP}=\beta_{LS}/\lambda_1^2$, we have that $\frac{\beta_{BP} \gamma_i^2}{1+\beta_{BP} \gamma_i^2} = \frac{\beta_{LS} \gamma_i^2}{\lambda_1^2+\beta_{LS} \gamma_i^2}$. Therefore,
\begin{align}
\label{Eq_observ3}
\sum \limits_{i=1}^{m} bias_{BP}^{2(i)} &= \sum \limits_{i=1}^{m}  \Big ( \frac{\beta_{BP} \gamma_i^2}{1+\beta_{BP} \gamma_i^2} \Big )^2 [\V^T\x]_i^2 \nonumber \\
&= \sum \limits_{i=1}^{m}  \Big ( \frac{\beta_{LS} \gamma_i^2}{\lambda_1^2+\beta_{LS} \gamma_i^2} \Big )^2 [\V^T\x]_i^2 \nonumber \\
& \leq \sum \limits_{i=1}^{m}  \Big ( \frac{\beta_{LS} \gamma_i^2}{\lambda_i^2+\beta_{LS} \gamma_i^2} \Big )^2 [\V^T\x]_i^2 = \sum \limits_{i=1}^{m} bias_{LS}^{2(i)}.
\end{align}
If $[\V^T\x]_i \neq 0$ for some indices $2 \leq i \leq m$, it is easy to see that the inequality is strict unless $\lambda_{{i}} = \lambda_1$ for these indices. 
Finally, recall that in the noiseless case the relation between $\sum \limits_{i=1}^{m} bias_{BP}^{2(i)}$ and $\sum \limits_{i=1}^{m} bias_{LS}^{2(i)}$, dictates the relation between $MSE_{BP}$ and $MSE_{LS}$.
}
\end{proof}

\tomtb{
Even though Observation~\ref{observ2} and Observation~\ref{observ3} consider the noiseless case, note that they cover events where the gap between $bias_{LS}^2$ and $bias_{BP}^2$ may be substantial enough to dictate the relationship between the MSEs also 
when the noise level is moderate.
For example, 
if $\beta_{BP}=\beta_{LS}$ and all the singular values are much smaller than 1 then the 'in particular'-part in Observation \ref{observ2} implies that $\sum \limits_{i=1}^{m} bias_{BP}^{2(i)}$ is much smaller than $\sum \limits_{i=1}^{m} bias_{LS}^{2(i)}$.
Another example, if $\beta_{BP}=\beta_{LS}/\lambda_1^2$ and the condition number of $\A\A^T$, i.e. the ratio $\lambda_1^2/\lambda_m^2$, is very large, then Observation \ref{observ3} implies that $\sum \limits_{i=1}^{m} bias_{BP}^{2(i)}$ is much smaller than $\sum \limits_{i=1}^{m} bias_{LS}^{2(i)}$.
}

\subsection{Discussion and implications for priors beyond $\ell_2$}
\label{sec_math_beyond}

As can be seen in \eqref{Eq_cost_typical_bias_var} and \eqref{Eq_cost_idbp_bias_var}, for the discussed Tikhonov regularization the bias term of each estimator is minimized if $\beta \to 0$, and in this case $bias_{LS}^2$ tends to $bias_{BP}^2$. This means that the performance gap in the noiseless case, which is stated in Observation~\ref{observ2} \tomtb{and Observation~\ref{observ3}}, tends to zero for $\beta \to 0$.
However, note that we consider here $\ell_2$ priors mainly as a surrogate to complex priors which are hard to analyze.
As we demonstrate in Section \ref{sec_exp}, the results that are obtained for sophisticated priors, such as TV, BM3D and DCGAN, indeed strongly correlate with the observations above 
\tomtb{ 
(especially with Observation~\ref{observ3} that implies 
an advantage of BP for badly conditioned $\A\A^T$). 
}
For such priors, the optimal value of $\beta$ for each fidelity term is significantly above 0 even in the noiseless case (contrary to $\ell_2$ priors), and the gap between the best recoveries is significant as well.

\tomtb{
Another motivation for connecting the above 
analysis 
to other priors comes from recognizing attributes that distinguish between the LS and BP fidelity terms regardless of the prior used with them.
}
\tomt{
Let us focus on the noiseless case, where $\y=\A\x$. In this case, \eqref{Eq_cost_typical} and \eqref{Eq_cost_bp} can be written as
\begin{align}
\label{Eq_cost_typical_noiseless}
f_{LS}(\tilde{\x}) &= \frac{1}{2} \| \A\x-\A\tilde{\x} \|_2^2 + \beta s(\tilde{\x})  \nonumber\\
&= \frac{1}{2} (\x-\tilde{\x})^T \A^T\A (\x-\tilde{\x}) + \beta s(\tilde{\x}), \\
\label{Eq_cost_bp_noiseless}
f_{BP}(\tilde{\x}) &= \frac{1}{2} \| \A^\dagger \A\x - \A^\dagger \A \tilde{\x} \|_2^2 + \beta s(\tilde{\x}) \nonumber\\
&=\frac{1}{2} (\x-\tilde{\x})^T \P_A (\x-\tilde{\x}) + \beta s(\tilde{\x}).
\end{align}
Under our SVD notations, we have $\A^T\A=\sum \limits_{i=1}^{m} \lambda_i^2 \v_i\v_i^T$ and $\P_A=\sum \limits_{i=1}^{m} \v_i\v_i^T$, where $\v_i$ is the right singular vector of $\A$ associated with the singular value $\lambda_i$. 
Therefore, we get
\begin{align}
\label{Eq_cost_typical_noiseless2}
f_{LS}(\tilde{\x}) &= \frac{1}{2} \sum \limits_{i=1}^{m} \lambda_i^2 | \v_i^T (\x-\tilde{\x})|^2 + \beta s(\tilde{\x}),  \\
\label{Eq_cost_bp_noiseless2}
f_{BP}(\tilde{\x}) &= \frac{1}{2} \sum \limits_{i=1}^{m} | \v_i^T (\x-\tilde{\x})|^2 + \beta s(\tilde{\x}).
\end{align}
}
\tomtb{
Note that $f_{BP}(\tilde{\x})$ equally weighs all $\{ | \v_i^T (\x-\tilde{\x})|^2 \}_{i=1}^m$,
contrary to $f_{LS}(\tilde{\x})$ that weighs them according to $\{  \lambda_i^2 \}$.
As in inverse problems one (typically) cares about minimizing the MSE, an {\em intuition} that minimizing \eqref{Eq_cost_bp_noiseless2} may have an advantage over minimizing \eqref{Eq_cost_typical_noiseless2} for {\em general} priors, comes from the similarity between the BP fidelity term and formulating the MSE as $\|\tilde{\x}-\x\|_2^2=\sum \limits_{i=1}^{n} | \v_i^T (\x-\tilde{\x})|^2$ (note that the sum here goes over all the $n$ basis vectors in $\V$).
For $\ell_2$ priors, we indeed have shown in Section~\ref{sec_math_perf} that this ``equal weighting'' strategy translates to the fact that $\{ bias_{BP}^{2(i)} \}$ do not depend on $\{  \lambda_i^2 \}$, contrary to $\{ bias_{LS}^{2(i)} \}$, which later yields the MSE advantage of BP over LS in Observation~\ref{observ3}. For $\ell_2$ priors, we have obtained analytical results and tradeoffs also for the noisy case. 
For other priors, we empirically show in Section~\ref{sec_exp} correlation to the above analytical findings.
}

An important factor that is not taken into account in the above analysis is optimization, since for $\ell_2$ priors there is a closed-form solution. Yet, for sophisticated priors iterative optimization schemes are inevitable, and the regularization parameter has an effect which is similar to the step size in these schemes. In such cases, extremely low value of $\beta$ inherently results in a massive slowdown in the convergence for convex priors \cite{hale2008fixed, giryes2018tradeoffs} and/or bad local minima for non-convex priors.
Taking a numerical optimization point of view, in the sequel we empirically show that
$\hat{\x}_{BP}$ is superior to $\hat{\x}_{LS}$ even for $\ell_2$ priors with $\beta \to 0$, if few iterations of conjugate gradients are used instead of the closed-form expressions \eqref{Eq_cost_typical_est} and \eqref{Eq_cost_idbp_est}. This implementation choice may be preferable in high-dimensional problems when it is not possible to invert the matrices. The advantage of BP in this case follows from the fact that the eigenvalues of $\P_A$ are only 1 (in the row space of $\A$) and 0 (in the null space of $\A$), while $\A^T\A$ may have very different eigenvalues in general, and conjugate gradients (among other methods) performs better when the 
eigenvalues are clustered \cite{kelley1995iterative}.
In Section \ref{sec_exp} we provide empirical evidence that BP requires less iterations than LS also for other optimization schemes and priors.

\section{Experiments with $\ell_2$ Priors}
\label{sec_math_ver}

In this section, we discuss the implications of the analytical results from Section \ref{sec_math} and verify them for specific observation models:
super-resolution and compressed sensing. \tomtb{In the first, all the singular values of $\A$ are smaller than 1 and the condition number of $\A\A^T$ is large, while in the latter it is possible that all singular values are greater than 1 and that the condition number is very moderate.}  
We also discuss the typical deblurring problem, 
\tomtb{which is highly ill-conditioned.} 
In this case, $\A^\dagger$ 
in $\hat{\x}_{BP}$ 
has to be regularized due to the large number of near zero singular values, and \eqref{Eq_cost_idbp_mse} needs to be modified accordingly.

\tomt{
Throughout this section, we use the closed-form estimators in \eqref{Eq_cost_typical_est} and \eqref{Eq_cost_idbp_est} to restore the images. The empirical performance of these two estimators is presented by markers, while the analytical expressions from \eqref{Eq_cost_typical_mse2} and \eqref{Eq_cost_bp_mse2} are plotted in solid curves. Different colors are used to distinguish between the two fidelity terms that are used for the estimation. 
}

\begin{figure}
  \centering
  \begin{subfigure}[b]{0.49\linewidth}
    \centering\includegraphics[width=100pt]{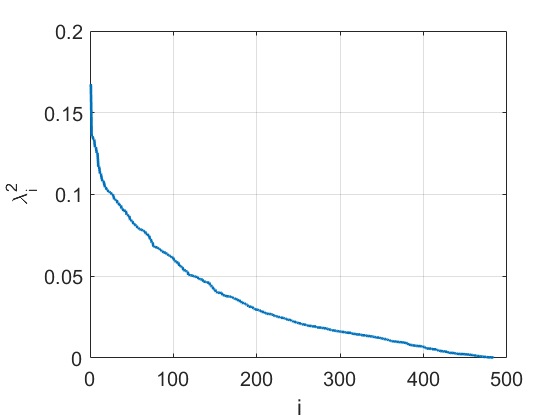}
    \caption{\label{fig:SR_eigenvalues}}
  \end{subfigure}%
  \begin{subfigure}[b]{0.49\linewidth}
    \centering\includegraphics[width=100pt]{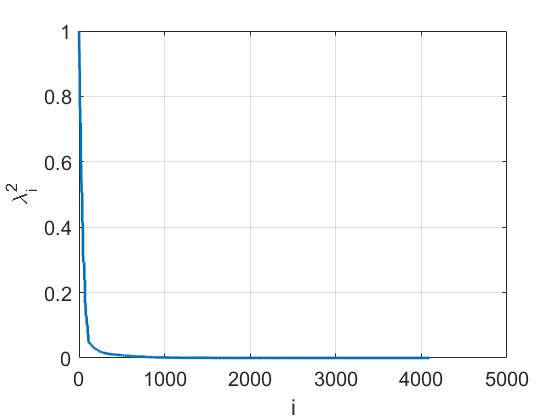}
    \caption{\label{fig:deb_eigenvalues}}
  \end{subfigure}%
\\
  \begin{subfigure}[b]{0.49\linewidth}
    \centering\includegraphics[width=100pt]{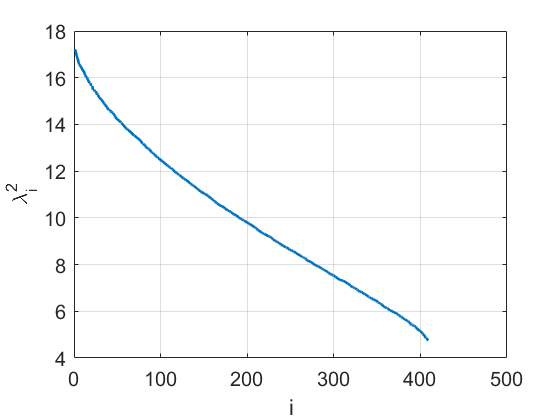}
    \caption{\label{fig:CS_eigenvalues}}
  \end{subfigure}
  \begin{subfigure}[b]{0.49\linewidth}
    \centering\includegraphics[width=100pt]{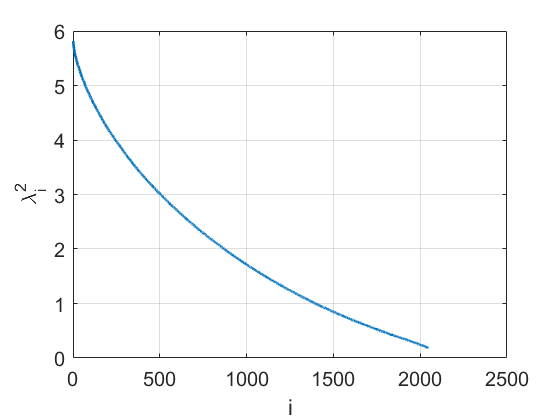}
    \caption{\label{fig:CS_eigenvalues2}}
  \end{subfigure}
  \caption{The (squared) singular values of $\A$ applied on a $64 \times 64$ image for: (\subref{fig:SR_eigenvalues}) SRx3 with $7 \times 7$ Gaussian filter \tomtb{($\frac{\lambda_1^2}{\lambda_m^2}=2.93\mathrm{e}3$)}; (\subref{fig:deb_eigenvalues}) blurring with $9 \times 9$ uniform filter \tomtb{($\frac{\lambda_1^2}{\lambda_m^2}=1.46\mathrm{e}7$)}; (\subref{fig:CS_eigenvalues}) CS with $m=0.1n$ Gaussian measurements and Haar basis \tomtb{($\frac{\lambda_1^2}{\lambda_m^2}=3.63$)}; (\subref{fig:CS_eigenvalues2}) CS with $m=0.5n$ Gaussian measurements and Haar basis \tomtb{($\frac{\lambda_1^2}{\lambda_m^2}=33.36$)}.}
\label{fig:eigenvalues}
\end{figure}

\subsection{Super-resolution}
\label{sec_math_ver_sr}

Let us consider the super-resolution (SR) task, 
where $\A$ is a composite operator of blurring (e.g. anti-aliasing filtering) followed by down-sampling. 
Note that the largest singular value of a typical low-pass filtering operation is 1, and it is associated with the DC \tomt{(i.e. the magnitude of the Fourier coefficient that is associated with zero frequency)}. The rest of the singular values are smaller than 1. 
The subsequent operator is subsampling, which inevitability reduces the energy of the signal (as $m<n$). Therefore, essentially, all the singular values of $\A$ are smaller than 1. 
\tomtb{Accordingly, the condition number of $\A\A^T$ is large.
These properties are} demonstrated in Fig. \ref{fig:SR_eigenvalues}  
for SR with scale factor 3 and 
Gaussian filter of size $7 \times 7$ and standard deviation 1.6 
(used in many works, e.g. \cite{dong2013nonlocally, tirer2018super, zhang2017learning}), which is performed on a $64 \times 64$ image (thus $n=4096$ and $m=484$).
We consider such a small image to allow computing the SVD of $\A$ (our analytic expressions require both $\{\lambda_i^2 \}_{i=1}^{m}$ and $\V$).  

We verify our analytical results for the SRx3 scenario mentioned above, and two cases: $\sigma_e=0$ and Gaussian noise with $\sigma_e=\sqrt{2}$. The experiments are performed on the {\em cameraman} image, resized to $64 \times 64$ pixels. In the noisy case, we average the results over 5 noise realizations. We have observed similar results for other images as well. 
We use the $\ell_2$ prior $s(\tilde{\x})=\frac{1}{2}\|\tilde{\x}\|_2^2$, which satisfies the assumptions ($\D=\I_n$ and $\gamma_i=1$).

The PSNR\footnote{The PSNR for a recovery $\hat{\x}$ of a uint8 image $\x \in \Bbb R^n$ is computed as $10\mathrm{log}_{10}\Big (\frac{255^2}{\frac{1}{n}\|\hat{\x}-\x\|_2^2} \Big )$.} results are presented in Fig. \ref{fig:SR_theory} and validate the analytical expressions. For $\sigma_e=0$, $\hat{\x}_{BP}$ is better than $\hat{\x}_{LS}$ for any value of the parameter $\beta$, as implied by Observation \ref{observ2} since all the singular values of $\A$ are smaller than 1 (Fig. \ref{fig:SR_eigenvalues}). 
\tomtb{
The rather large gap in favor of BP also agrees with Observation~\ref{observ3} that predicts it when the ratio $\lambda_1^2/\lambda_m^2$ is large.
The fact that BP at $\beta/\lambda_1^2=5.97\beta$ outperforms LS at $\beta$, further verifies Observation~\ref{observ3}.} 
For $\sigma_e=\sqrt{2}$, the gap between the estimators is reduced because $var_{BP}$  
is worse than $var_{LS}$ at handling the small singular values, as mentioned in Observation \ref{observ1}. 

To demonstrate the numerical optimization advantage of the BP cost over the LS cost for $\beta \to 0$ (where the gap between the bias terms in \eqref{Eq_cost_typical_bias_var} and \eqref{Eq_cost_idbp_bias_var} tends to 0), we repeat the experiments above for very small values of $\beta$. However, this time instead of inverting the matrices in \eqref{Eq_cost_typical_est} and \eqref{Eq_cost_idbp_est} we obtain the estimators using the conjugate gradient method. The results are presented in Fig. \ref{fig:SR_theory_cg}. Remarkably, a single iteration is enough for obtaining the exact BP estimator (for $\ell_2$ prior).

\begin{figure}
  \centering
  \begin{subfigure}[b]{0.5\linewidth}
    \centering\includegraphics[width=120pt]{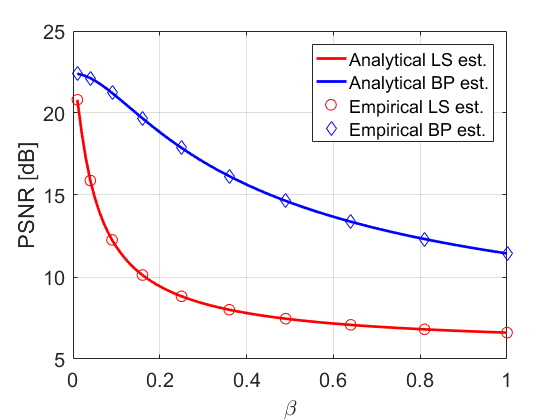}
    \caption{\label{fig:SR_noiseless}}
  \end{subfigure}%
  \begin{subfigure}[b]{0.5\linewidth}
    \centering\includegraphics[width=120pt]{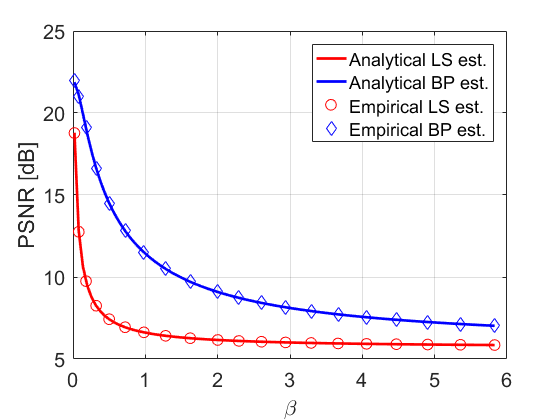}
    \caption{\label{fig:SR_noisy}}
  \end{subfigure}
  \caption{Super-resolution with Gaussian filter and scale factor of 3, using $\ell_2$ prior. PSNR (for {\em cameraman}) vs. $\beta$ (regularization parameter), for (\subref{fig:SR_noiseless}) $\sigma_e=0$, and (\subref{fig:SR_noisy}) $\sigma_e=\sqrt{2}$.}
\label{fig:SR_theory}

\vspace{4mm}

  \centering
  \begin{subfigure}[b]{0.5\linewidth}
    \centering\includegraphics[width=120pt]{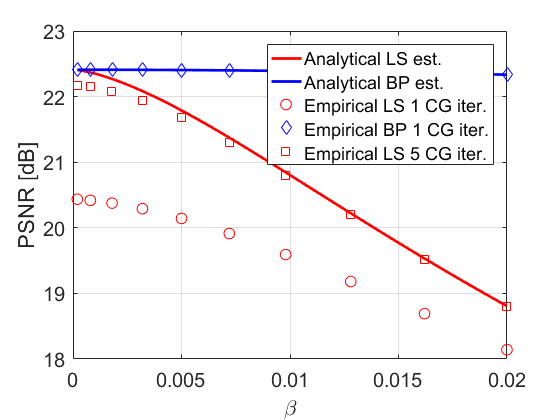}
    \caption{\label{fig:SR_noiseless_cg}}
  \end{subfigure}%
  \begin{subfigure}[b]{0.5\linewidth}
    \centering\includegraphics[width=120pt]{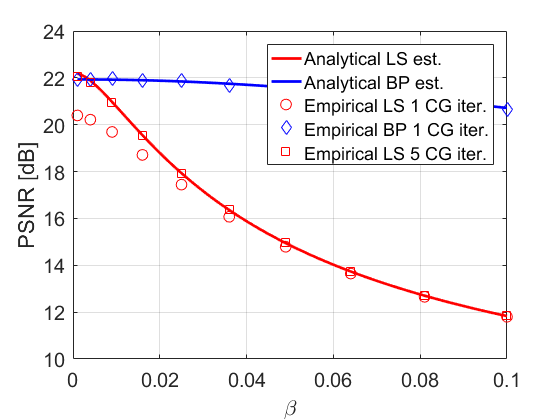}
    \caption{\label{fig:SR_noisy_cg}}
  \end{subfigure}
  \caption{Super-resolution with Gaussian filter and scale factor of 3, using $\ell_2$ prior and iterations of conjugate gradients instead of matrix inversion. PSNR (for {\em cameraman}) vs. $\beta$ (regularization parameter), for (\subref{fig:SR_noiseless_cg}) $\sigma_e=0$, and (\subref{fig:SR_noisy_cg}) $\sigma_e=\sqrt{2}$. 
Note that the LS cost requires more CG iterations than the BP cost to attend the solution (solid line).
}
\label{fig:SR_theory_cg}
\end{figure}

\begin{figure}
  \centering
  \begin{subfigure}[b]{0.5\linewidth}
    \centering\includegraphics[width=120pt]{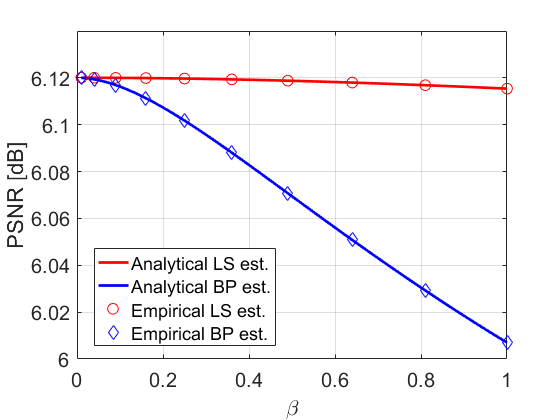}
    \caption{\label{fig:CS_0p2}}
  \end{subfigure}%
  \begin{subfigure}[b]{0.5\linewidth}
    \centering\includegraphics[width=120pt]{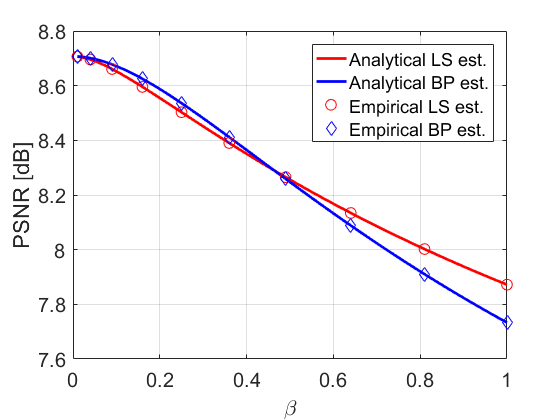}
    \caption{\label{fig:CS_0p5}}
  \end{subfigure}
  \caption{Compressed sensing with Gaussian measurements and Haar basis, using $\ell_2$ prior. PSNR (for {\em cameraman}) vs. $\beta$ (regularization parameter), for (\subref{fig:CS_0p2}) $m=0.1n$, and (\subref{fig:CS_0p5}) $m=0.5n$.}
\label{fig:CS_theory}

\vspace{4mm}

  \centering
  \begin{subfigure}[b]{0.5\linewidth}
    \centering\includegraphics[width=120pt]{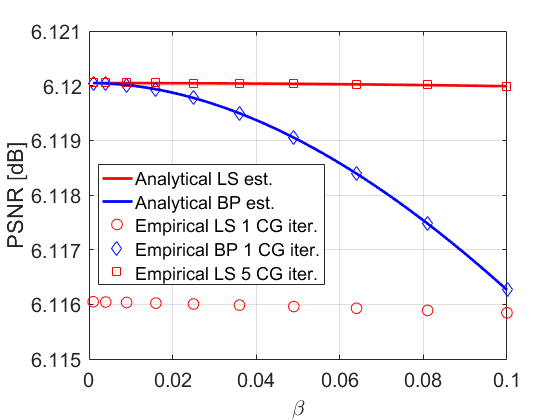}
    \caption{\label{fig:CS_0p2_cg}}
  \end{subfigure}%
  \begin{subfigure}[b]{0.5\linewidth}
    \centering\includegraphics[width=120pt]{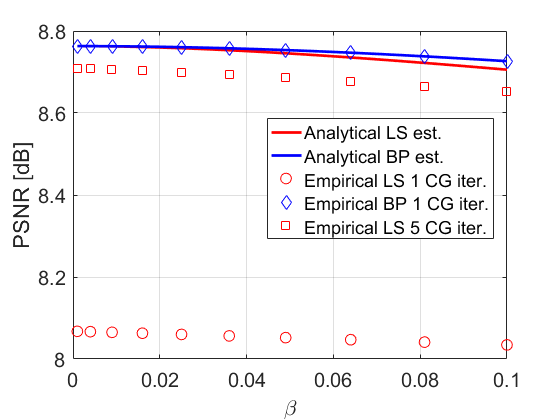}
    \caption{\label{fig:CS_0p5_cg}}
  \end{subfigure}
  \caption{Compressed sensing with Gaussian measurements and Haar basis, using $\ell_2$ prior and iterations of conjugate gradients instead of matrix inversion. PSNR (for {\em cameraman}) vs. $\beta$ (regularization parameter), for (\subref{fig:CS_0p2_cg}) $m=0.1n$, and (\subref{fig:CS_0p5_cg}) $m=0.5n$. 
Note that the LS cost requires more CG iterations than the BP cost to attend the solution (solid line).
}
\label{fig:CS_theory_cg}
\end{figure}

\subsection{Compressed sensing}
\label{sec_math_ver_cs}

Contrary to SR scenarios, in compressed sensing (CS) \tomtb{the condition number of $\A\A^T$ is moderate} and the singular values of $\A$ may be larger than 1. Consider the commonly examined scenario where $\A$ is the multiplication of an $m \times n$ Gaussian measurement matrix (whose i.i.d. entries are drawn from $\mathcal{N}(0,1/m)$) with an $n \times n$ Haar wavelet basis. We have observed that for high compression, e.g. $m/n=0.1$, all the singular values are larger than 1 \tomtb{and the condition number is very small}, as demonstrated in Fig.~\ref{fig:CS_eigenvalues}. However, for lower compression, e.g. $m/n=0.5$, there are also singular values smaller than 1 \tomtb{and the condition number increases}, as demonstrated in Fig.~\ref{fig:CS_eigenvalues2}.

We verify our analytical results for these two 
compression ratios (both with $\sigma_e=0$). 
The experiments are performed on the same $64 \times 64$ version of {\em cameraman} image, and we use again the $\ell_2$ prior $s(\tilde{\x})=\frac{1}{2}\|\tilde{\x}\|_2^2$.
The results are presented in Fig. \ref{fig:CS_theory} and validate the analytical expressions. For $m/n=0.1$, $\hat{\x}_{LS}$ is better than $\hat{\x}_{BP}$ for any value of $\beta$, as implied by Observation \ref{observ2} since all the singular values of $\A$ are greater than 1 (Fig. \ref{fig:CS_eigenvalues}). 
\tomtb{
To verify Observation~\ref{observ3} in this case, see that BP at $\beta/\lambda_1^2=0.058\beta$ has (slightly) higher PSNR than LS at $\beta$, e.g. for $\beta=1$. We have verified this also for very large values of $\beta$ (not presented here)---both curves decrease and reach a similar plateau at high $\beta$, yet BP at $\beta/\lambda_1^2$ indeed has higher PSNR than LS at $\beta$, but the difference is extremely small.} 
Interestingly, for $m/n=0.5$, where some singular values of $\A$ are smaller than 1 (Fig. \ref{fig:CS_eigenvalues2}), $\hat{\x}_{BP}$ gets better results than $\hat{\x}_{LS}$.
\tomtb{
Also in this scenario, it can be verified that BP at $\beta/\lambda_1^2=0.171\beta$ has higher PSNR than LS at $\beta$, as implied by Observation~\ref{observ3}.
The fact that the gap between BP and LS for $m/n=0.1$ and $m/n=0.5$ has been changed in favor of BP in the latter agrees with the derivation of Observation~\ref{observ3} in \eqref{Eq_observ3} that links the advantage of BP to an increased $\lambda_1^2/\lambda_m^2$ ratio.
}

We demonstrate again the numerical optimization advantage of the BP cost over the LS cost by repeating the experiments above for very small values of $\beta$, while using the conjugate gradient method instead of matrix inversion.
The results are presented in Fig. \ref{fig:CS_theory_cg}. It can be seen again that for the $\ell_2$ prior a single iteration is enough for obtaining the exact BP estimator.

We find it necessary to emphasize that compressed sensing scenarios require a sparsity-inducing prior, e.g. $s(\tilde{\x})=\|\tilde{\x}\|_1$ \tomt{or TV prior}, rather than an $\ell_2$ prior, for which both estimators exhibit poor results (i.e. very low PSNR). 
However, 
our purpose here is merely to validate our analysis, which applies only to $\ell_2$ priors, 
\tomtb{for a case in which all the singular values are greater than 1 and/or the condition number is small.}

Finally, note that for Gaussian $\A$ there is no efficient way to implement the operators $\A$ and $\A^T$ for large dimensions. 
Therefore, in practice, taking $\A$ to be the subsampled Fourier transform is more common, e.g. in sparse MRI \cite{lustig2007sparse}.
However, note that for this acquisition model $\A^\dagger$ is simply the Hermitian transpose of $\A$ (this property follows from the fact that the subsampled Fourier transform is a tight frame \cite{kovavcevic2008introduction}), which together with the unitarity of the Fourier transform 
leads to $\| \A^\dagger ( \y - \A\tilde{\x}) \|_2^2 = \| \y-\A\tilde{\x} \|_2^2$. This means that the two cost functions coincide, which is also implied by the fact that in this case all the singular values of $\A$ are 1 and thus \eqref{Eq_cost_typical_mse} is identical to \eqref{Eq_cost_idbp_mse}. Therefore, 
we do not make a comparison for this case.

\begin{figure}
  \centering
  \begin{subfigure}[b]{0.5\linewidth}
    \centering\includegraphics[width=120pt]{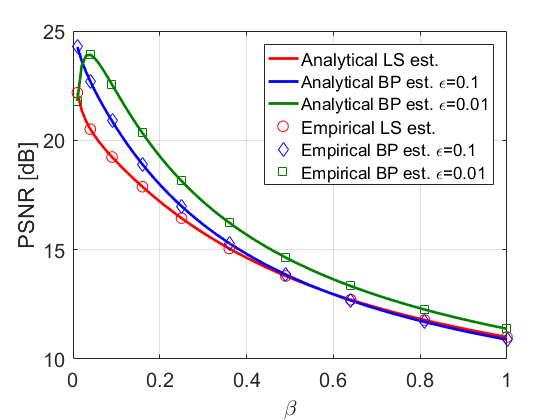}
    \caption{\label{fig:deb_noiseless}}
  \end{subfigure}%
  \begin{subfigure}[b]{0.5\linewidth}
    \centering\includegraphics[width=120pt]{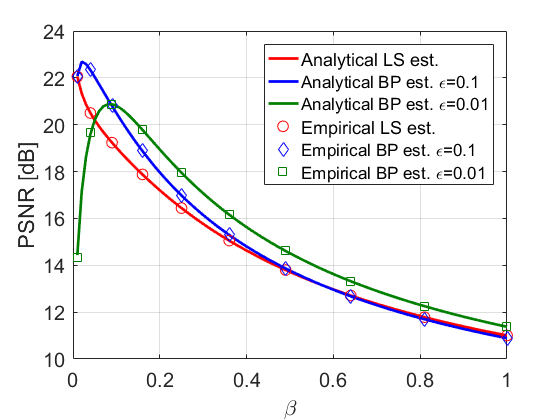}
    \caption{\label{fig:deb_noisy}}
  \end{subfigure}
  \caption{Deblurring with uniform $9 \times 9$ blur kernel, using $\ell_2$ prior. PSNR (for {\em cameraman}) vs. $\beta$ (regularization parameter), for (\subref{fig:deb_noiseless}) $\sigma_e=\sqrt{0.3}$, and (\subref{fig:deb_noisy}) $\sigma_e=\sqrt{2}$.}
\label{fig:deb_theory}
\end{figure}

\begin{figure}
  \centering
  \begin{subfigure}[b]{0.5\linewidth}
    \centering\includegraphics[width=120pt]{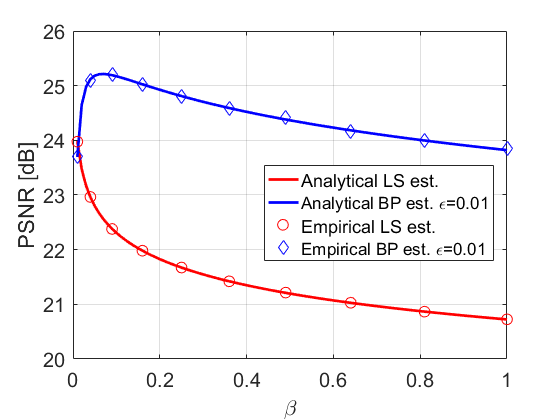}
    \caption{\label{fig:deb_dev_circ}}
  \end{subfigure}%
  \begin{subfigure}[b]{0.5\linewidth}
    \centering\includegraphics[width=120pt]{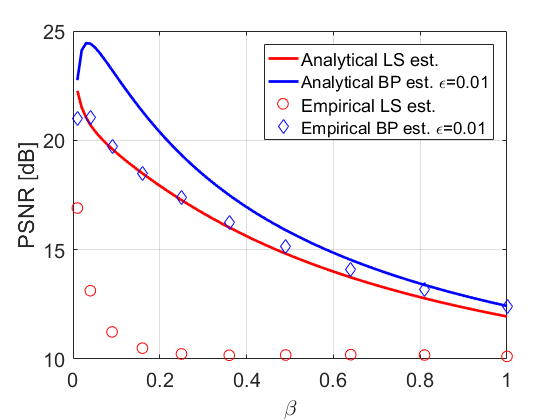}
    \caption{\label{fig:deb_dev_noncirc}}
  \end{subfigure}
  \caption{Deblurring with uniform $9 \times 9$ blur kernel and $\sigma_e=\sqrt{0.3}$, using priors with different $\D^T\D$. PSNR (for {\em cameraman}) vs. $\beta$ (regularization parameter), for (\subref{fig:deb_dev_circ}) circulant $\D^T\D$, and (\subref{fig:deb_dev_noncirc}) non-circulant $\D^T\D$. Note that $\D^T\D \neq \V \bGamma^2 \V^T$ in (\subref{fig:deb_dev_noncirc}).}
\label{fig:deb_dev_theory}

\vspace{4mm}

  \centering
  \begin{subfigure}[b]{0.5\linewidth}
    \centering\includegraphics[width=120pt]{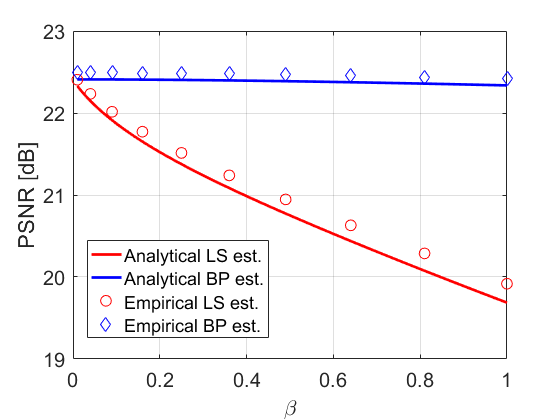}
    \caption{\label{fig:SR_dev_circ}}
  \end{subfigure}%
  \begin{subfigure}[b]{0.5\linewidth}
    \centering\includegraphics[width=120pt]{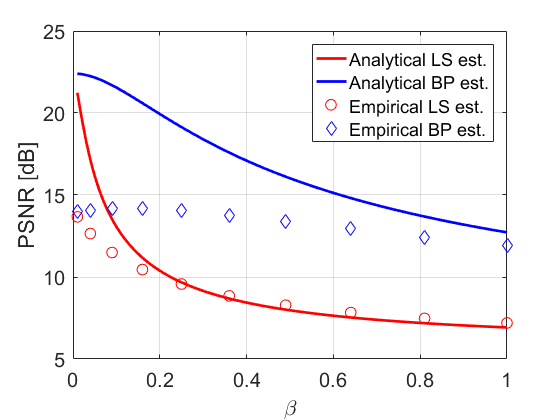}
    \caption{\label{fig:SR_dev_noncirc}}
  \end{subfigure}
  \caption{Super-resolution with Gaussian filter and scale factor of 3 and $\sigma_e=0$, using priors with different $\D^T\D$. PSNR (for {\em cameraman}) vs. $\beta$ (regularization parameter), for (\subref{fig:SR_dev_circ}) circulant $\D^T\D$, and (\subref{fig:SR_dev_noncirc}) non-circulant $\D^T\D$. Note that $\D^T\D \neq \V \bGamma^2 \V^T$ in (\subref{fig:SR_dev_circ}) and (\subref{fig:deb_dev_noncirc}).}
\label{fig:SR_dev_theory}
\end{figure}

\subsection{Deblurring}
\label{sec_math_ver_deb}

In the deblurring problem, $\A$ is a square ($m=n$) ill-conditioned matrix that performs blurring (i.e. filtering by a blur kernel).  
Typically, the blur kernel coefficients are normalized such that their sum is 1. Thus, the largest singular value of $\A$ is 1 (associated with the DC), and many other singular values are near 0.
\tomtb{Accordingly, the condition number of $\A\A^T$ is extremely large. These properties are} 
demonstrated in Fig. \ref{fig:deb_eigenvalues} for uniform kernel of size $9 \times 9$ (used in many works, e.g. \cite{guerrero2008image, danielyan2012bm3d, tirer2019image}).

Note that if one uses $\hat{\x}_{BP}$, exactly as defined in \eqref{Eq_cost_idbp_est}, \tomtb{both Observation~\ref{observ2} and Observation~\ref{observ3} imply an advantage of $\hat{\x}_{BP}$ over $\hat{\x}_{LS}$ in the noiseless case due to small singular values and a large condition number, respectively.}  
However, since in the deblurring problem $\A$ is not rank-deficient but rather (very) ill-conditioned, deblurring scenarios always assume that the measurements are noisy (typically with low noise levels).
Therefore, it is required to regularize the inversion of $\A\A^T$ in $\A^\dagger$ in order to mitigate the effect of near zero $\{\lambda_i\}$ on the variance of $\hat{\x}_{BP}$. 
A common regularized inversion is diagonal loading: inverting $\A\A^T+\epsilon\I_n$ instead of $\A\A^T$, where $\epsilon$ is a parameter. This is equivalent to replacing $\lambda_i^2$ with $\lambda_i^2+\epsilon$ 
in the eigen-decomposition of $\A\A^T$.

For $\hat{\x}_{BP}$ with such a regularized inversion, it is not hard to repeat the computations in \eqref{Eq_cost_idbp_mse} and obtain a very similar result, where $1_{i \leq m}$ is replaced with $\lambda_i^2/(\lambda_i^2+\epsilon)$ and $\lambda_i^{-2}1_{i \leq m}$ is replaced with $\lambda_i^2/(\lambda_i^2+\epsilon)^2$. Formally, we get 
\begin{align}
\label{Eq_cost_idbp_mse_diag_load}
MSE_{BP} &  = \sum \limits_{i=1}^{n} \Big ( \frac{\lambda_i^2/(\lambda_i^2+\epsilon)}{\lambda_i^2/(\lambda_i^2+\epsilon)+\beta \gamma_i^2} - 1 \Big )^2 [\V^T\x]_i^2 \nonumber \\
& \hspace{10pt} + \sigma_e^2 \sum \limits_{i=1}^{n} \frac{\lambda_i^2/(\lambda_i^2+\epsilon)^2}{(\lambda_i^2/(\lambda_i^2+\epsilon)+\beta \gamma_i^2)^2}, \nonumber \\
&  = \sum \limits_{i=1}^{n} \Big ( \frac{\beta \gamma_i^2}{\lambda_i^2/(\lambda_i^2+\epsilon)+\beta \gamma_i^2} \Big )^2 [\V^T\x]_i^2  \nonumber \\
& \hspace{10pt} + \sigma_e^2 \sum \limits_{i=1}^{n} \frac{1}{\lambda_i^2(1+\beta \gamma_i^2(\lambda_i^2+\epsilon)/\lambda_i^2)^2}.
\end{align}
Therefore, as could be expected, increasing the amount of regularization $\epsilon$ reduces the variance of $\hat{\x}_{BP}$ but increases its bias. 
As a sanity check, observe that for $\epsilon \to 0$ we get that \eqref{Eq_cost_idbp_mse_diag_load} coincides with \eqref{Eq_cost_idbp_mse} (recall $m=n$).
Since in this case the performance of $\hat{\x}_{BP}$ depends on the couple $(\beta,\epsilon)$, we cannot obtain clear properties like 
\tomtb{the observations in Section~\ref{sec_math_perf}} 
that hold uniformly for any parameter setting. 
Yet, as demonstrated below and in the sequel, we have empirically observed that it is possible to find settings of $(\beta,\epsilon)$ that balance the bias and variance of $\hat{\x}_{BP}$ and therefore lead to very good results despite the observed noise.

We verify \eqref{Eq_cost_idbp_mse_diag_load} for the uniform blur kernel mentioned above, and two levels of Gaussian noise: $\sigma_e=\sqrt{0.3}$ and $\sigma_e=\sqrt{2}$. The experiments are performed on the $64 \times 64$ version of {\em cameraman} image, and we use again the $\ell_2$ prior. 
The results are presented in Fig. \ref{fig:deb_theory}. They show that $\hat{\x}_{BP}$ with good tuning of $(\beta,\epsilon)$ can outperform $\hat{\x}_{LS}$, especially when the noise level is low. This implies that ''well-tuned'' $\hat{\x}_{BP}$ 
handles the 
\tomtb{badly conditioned} 
$\A$ (Fig. \ref{fig:deb_eigenvalues}) better than $\hat{\x}_{LS}$.

\subsection{\tomt{The effect of the joint right singular vectors assumption}}
\label{sec_math_ver_deviation}

\tomt{
In this section, we compare the empirical MSE and the analytical formulas in \eqref{Eq_cost_typical_mse2}, \eqref{Eq_cost_bp_mse2} and \eqref{Eq_cost_idbp_mse_diag_load} in cases where the condition $\D^T\D=\V \bGamma^2 \V^T$ is violated (recall that the columns of $\V$ are the right singular vectors of $\A$ and $\bGamma^2$ is a diagonal matrix, as defined in Section~\ref{sec_math_assump}). 
Since our formulas require the diagonal of $\bGamma^2$ (i.e. $\{ \gamma_i^2 \}$), we compute it as the diagonal of $\V^T\D^T\D\V$, which is exact under the analysis assumption and can be regarded as an approximation when $\D^T\D \neq \V \bGamma^2 \V^T$.
}

\tomt{
We start with examining the case of $\D^T\D = \bOmega_{DIF}^T\bOmega_{DIF}+0.01\I_n$, where $\bOmega_{DIF}$ is the 2D finite difference operator and the diagonal loading is required to make $\D^T\D \succ 0$.
Note that for the deblurring task we have that both $\A$ and $\D^T\D$ are circulant matrices that can be diagonalized by the DFT matrix. Therefore, the condition $\D^T\D = \V \bGamma^2 \V^T$ holds for $\V$ that equals the (inverse) DFT matrix.
However, for the SR task $\A$ cannot be singularly decomposed by a Fourier basis. Therefore, the condition cannot be satisfied.
}

\tomt{
We repeat previous deblurring and SR experiments with the examined $\D^T\D$.
The results are presented in Figs.~\ref{fig:deb_dev_circ} and \ref{fig:SR_dev_circ}. For deblurring we see perfect agreement between the empirical results and the analytical formulas (as expected). For SR we see that violating the condition has led to a small gap between the empirical results and the formulas.
}

\tomt{
Now, we further increase the violation of the condition by breaking the circularity property of $\D^T\D$. We do it by replacing $\bOmega_{DIF}$ with a non-circulant operator $\tilde{\bOmega}$ that performs finite difference only on every 8th pixel (and identity on the rest).
We repeat the previous deblurring and SR experiments and present the results in Figs.~\ref{fig:deb_dev_noncirc} and \ref{fig:SR_dev_noncirc}.
It is easy to see that the deviation of the empirical results from the formulas further grows for both tasks. 
}

\tomt{
The experiments demonstrate that the deviation between the empirical MSE and the analytical expressions is proportional to how much the condition on $\D^T\D$ is violated. 
Yet, the overall trend in the curves still shares similarity with the analytical results, 
which motivates considering the observations obtained by the analytical analysis for practical sophisticated priors.
}

\begin{figure}
  \centering
  \begin{subfigure}[b]{0.5\linewidth}
    \centering\includegraphics[width=120pt]{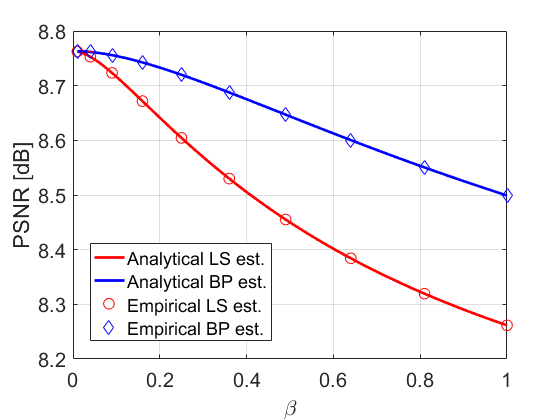}
    \caption{\label{fig:CS_0p5_prior1}}
  \end{subfigure}%
  \begin{subfigure}[b]{0.5\linewidth}
    \centering\includegraphics[width=120pt]{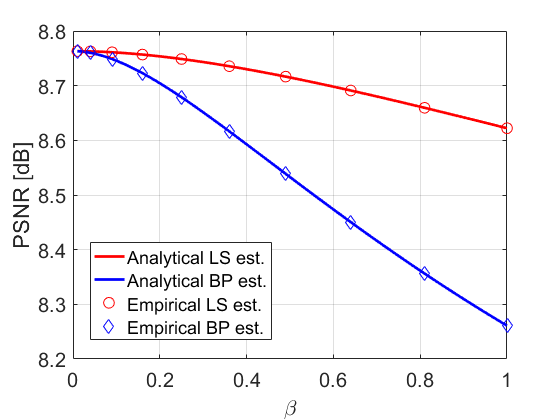}
    \caption{\label{fig:CS_0p5_prior2}}
  \end{subfigure}
  \caption{Compressed sensing with $m=0.5n$ Gaussian measurements and Haar basis, using $\ell_2$ prior. PSNR (for $\x$ that is the projection of {\em cameraman} onto $\mathcal{W}^\perp$) vs. $\beta$ (regularization parameter), for (\subref{fig:CS_0p5_prior1}) $\mathcal{W}$ that is spanned by the columns of $\V(:,1:m-500)$, and (\subref{fig:CS_0p5_prior2}) $\mathcal{W}$ that is spanned by the columns of $\V(:,1000:m)$.}
\label{fig:CS_theory_prior}
\end{figure}

\subsection{\tomt{Incorporating prior knowledge on $\x$ with the results}}
\label{sec_math_ver_subspace_prior}

\tomt{
The analytical MSE formulas in \eqref{Eq_cost_typical_mse2} and \eqref{Eq_cost_bp_mse2} are conditioned on the latent image $\x$, as the expectations are taken only with respect to the noise $\e$.
These expressions have led to} 
\tomtb{observations in Section~\ref{sec_math}} 
\tomt{
that depend only the singular values of $\A$ (i.e. $\{\lambda_i\}$) and do not require prior knowledge on $\x$ (recall that accurately modeling natural images is difficult).
The usefulness of these observations for preferring one fidelity term over the other for sophisticated priors is demonstrated in Section~\ref{sec_exp}.
}

\tomt{
However, a natural question arises: How can one leverage prior knowledge on $\x$ to improve the criterion for choosing the fidelity term?
}

\tomt{
In this section we briefly demonstrate, using a controlled experiment, how the observations in Section~\ref{sec_math} can be polished given a constraint that $\x$ resides in $\mathcal{W}^\perp$ the orthogonal complement of a {\em known} subspace $\mathcal{W}$. Note that for low-dimensional $\mathcal{W}$ such that $m<n-\mathrm{dim}(\mathcal{W})$, we still have an ill-posed linear inverse problem.
}

\tomt{
We consider the compressed sensing scenario from Section~\ref{sec_math_ver_cs} where $n=64^2$, $m/n=0.5$ and $\sigma_e=0$. 
In this case $\A$ has (more than 1000) singular values that are larger than 1 and (slightly more than 500) singular values that are smaller than 1 (see Fig.~\ref{fig:CS_eigenvalues2}). 
Therefore, the events in the 'in particular'-part in Observation \ref{observ2} do not occur.
Indeed, observe that in Fig.~\ref{fig:CS_0p5} none of the estimators is consistently (i.e. for any $\beta$) better than the other when $\x$ is the {\em cameraman} image.
}

\tomt{
Now, let us use the notation from Section~\ref{sec_math}, where the columns of $\V$, that is, the right singular vectors of $\A$, are ordered according to a descending order of the singular values (from 1 to $m$), and the last $n-m$ columns span the null space of $\A$.
Suppose that $\mathcal{W}$ is the subspace spanned by the columns of $\V(:,1:m-500)$, where we use Matlab notation, 
and that $\x \in \mathcal{W}^\perp$. Due to the orthogonality of $\V$, we have that $[\V^T\x]_i=0$ for any $1 \leq i \leq m-500$. 
Substituting this property in \eqref{Eq_cost_typical_bias_var} and \eqref{Eq_cost_idbp_bias_var}, we get
\begin{align}
\label{Eq_cost_bias_prior}
bias_{LS}^2 &=  \sum \limits_{i=m-499}^{m} bias_{LS}^{2(i)}  +  \sum \limits_{i=m+1}^{n} [\V^T\x]_i^2 , \nonumber \\
bias_{BP}^2 &=  \sum \limits_{i=m-499}^{m} bias_{BP}^{2(i)} +  \sum \limits_{i=m+1}^{n} [\V^T\x]_i^2.
\end{align}
Therefore, for the considered CS scenario, we have that $bias_{BP}^2<bias_{LS}^2$ for any $\beta$ (because $\lambda_i<1$ for all $m-499 \leq i \leq m$). Since $\sigma_e=0$, this implies that $MSE_{BP} < MSE_{LS}$ for any $\beta$. 
}

\tomt{
Note that for $\mathcal{W}$ that is the subspace spanned by the columns of $\V(:,1000:m)$ and $\x \in \mathcal{W}^\perp$ (i.e. $\x$ in a subspace spanned by columns of $\V$ that are either associated with singular values that are greater than 1 or with the null space of $\A$), similar arguments lead to $MSE_{BP} > MSE_{LS}$ for any $\beta$.
Fig.~\ref{fig:CS_theory_prior} verifies both results for a test image $\x$ that is the projection of the {\em cameraman} image onto $\mathcal{W}^\perp$ (i.e. $\x = \P_{\mathcal{W}^\perp}\x_0$, where $\x_0$ is the {\em cameraman} image).
}

\tomt{
Note that the behavior in Fig.~\ref{fig:CS_theory_prior} cannot be predicted by the 'in particular'-part in Observation \ref{observ2} that considers {\em all} the singular values of $\A$, {\em regardless} of $\x$.
We believe that a detailed study with constraints on $\x$ that better fit images is an interesting direction for future research. 
}

\begin{figure}[t]
  \centering
  \begin{subfigure}[b]{0.5\linewidth}
    \centering\includegraphics[width=120pt]{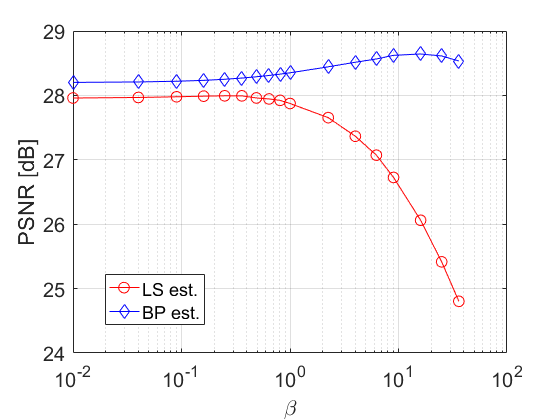}
    \caption{\label{fig:SR_tv_noiseless}}
  \end{subfigure}%
  \begin{subfigure}[b]{0.5\linewidth}
    \centering\includegraphics[width=120pt]{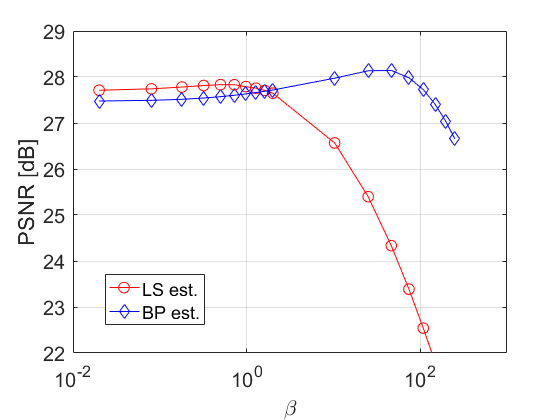}
    \caption{\label{fig:SR_tv_noisy}}
  \end{subfigure}
  \caption{Super-resolution with Gaussian filter and scale factor of 3, using TV prior and 100 iterations of FISTA. PSNR (averaged over 8 test images) vs. $\beta$ (regularization parameter), for (\subref{fig:SR_tv_noiseless}) $\sigma_e=0$, and (\subref{fig:SR_tv_noisy}) $\sigma_e=\sqrt{2}$.}
\label{fig:SR_tv}

\vspace{2mm}

  \centering
  \begin{subfigure}[b]{0.5\linewidth}
    \centering\includegraphics[width=120pt]{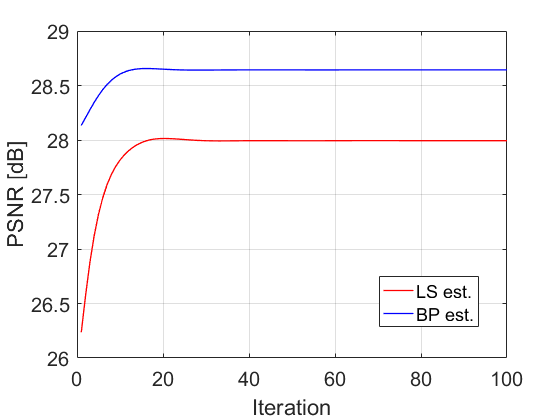}
    \caption{\label{fig:SR_tv_noiseless2}}
  \end{subfigure}%
  \begin{subfigure}[b]{0.5\linewidth}
    \centering\includegraphics[width=120pt]{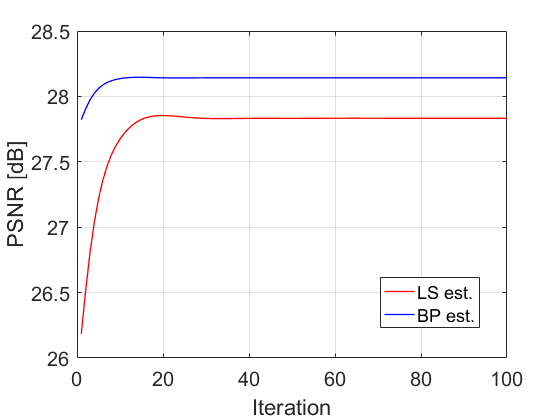}
    \caption{\label{fig:SR_tv_noisy2}}
  \end{subfigure}
  \caption{Super-resolution with Gaussian filter and scale factor of 3, using TV prior. PSNR (for best uniform setting of $\beta$, averaged over 8 test images) vs. FISTA iteration number, for (\subref{fig:SR_tv_noiseless2}) $\sigma_e=0$, and (\subref{fig:SR_tv_noisy2}) $\sigma_e=\sqrt{2}$.}
\label{fig:SR_tv_psnr_vs_iter}
\end{figure}

\begin{figure}[t]
  \centering
  \begin{subfigure}[b]{0.5\linewidth}
    \centering\includegraphics[width=120pt]{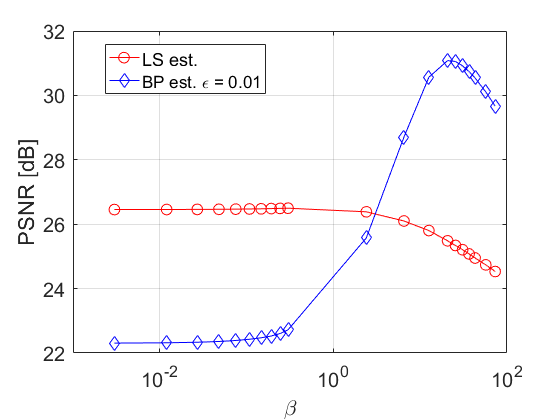}
    \caption{\label{fig:deb_tv_noiseless}}
  \end{subfigure}%
  \begin{subfigure}[b]{0.5\linewidth}
    \centering\includegraphics[width=120pt]{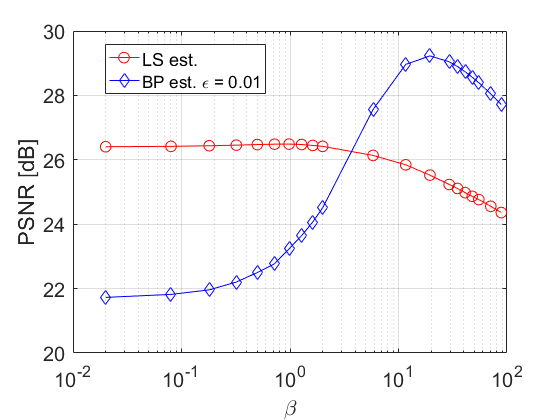}
    \caption{\label{fig:deb_tv_noisy}}
  \end{subfigure}
  \caption{Deblurring with uniform $9 \times 9$ blur kernel, using TV prior and 100 iterations of FISTA. PSNR (averaged over 8 test images) vs. $\beta$ (regularization parameter), for (\subref{fig:deb_tv_noiseless}) $\sigma_e=\sqrt{0.3}$, and (\subref{fig:deb_tv_noisy}) $\sigma_e=\sqrt{2}$.}
\label{fig:deb_tv}

\vspace{5.5mm}

  \centering
  \begin{subfigure}[b]{0.5\linewidth}
    \centering\includegraphics[width=120pt]{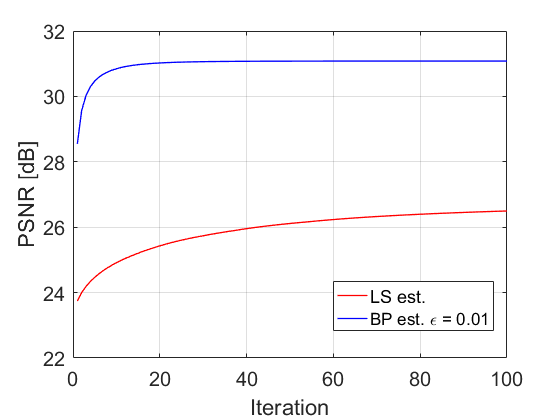}
    \caption{\label{fig:deb_tv_noiseless2}}
  \end{subfigure}%
  \begin{subfigure}[b]{0.5\linewidth}
    \centering\includegraphics[width=120pt]{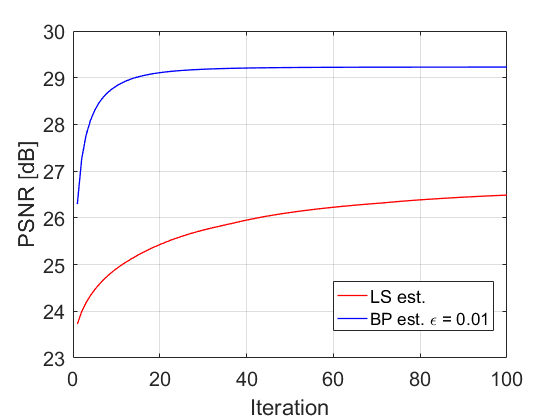}
    \caption{\label{fig:deb_tv_noisy2}}
  \end{subfigure}
  \caption{Deblurring with uniform $9 \times 9$ blur kernel, using TV prior. PSNR (for best uniform setting of $\beta$, averaged over 8 test images) vs. FISTA iteration number, for (\subref{fig:deb_tv_noiseless2}) $\sigma_e=\sqrt{0.3}$, and (\subref{fig:deb_tv_noisy2}) $\sigma_e=\sqrt{2}$.}
\label{fig:deb_tv_psnr_vs_iter}
\end{figure}

\section{Experiments with Sophisticated Priors}
\label{sec_exp}

In this section we empirically demonstrate that the behavior of $\hat{\x}_{BP}$ and $\hat{\x}_{LS}$ (the minimizers of $f_{BP}(\tilde{\x})$ and $f_{LS}(\tilde{\x})$) for sophisticated convex and non-convex priors (for whom mathematical analysis is hard or even intractable) strongly correlates with properties for which we have established concrete mathematical reasoning in the case of $\ell_2$ priors.
Specifically, 
\tomtb{for super-resolution and deblurring tasks 
(where the condition number of $\A\A^T$ is very large) 
BP cost function can lead to significantly improved results compared to the LS cost function, 
yet, there is inverse proportion between the performance gap and the noise level 
(since the singular values of $\A$ are small in these tasks)}. 
\tomtb{
For Gaussian compressed sensing with 
low $m/n$ ratio (where the condition number is small and the singular values are greater than 1) $\hat{\x}_{BP}$ is not significantly better than $\hat{\x}_{LS}$, but it is quite robust to noise. However, when the $m/n$ ratio increases (then the condition number increases and some singular values are smaller than 1) the advantage of BP is more significant, but inversely proportional to the noise level.
}

\subsection{TV prior}
\label{sec_exp_tv}

We start with the widely-used (isotropic) total-variation (TV) prior \cite{rudin1992nonlinear}, which is given by
\begin{align}
\label{Eq_tv_prior}
s(\tilde{\x})=0.1 \sum\limits_{i,j} \sqrt{ |\tilde{x}_{i+1,j} - \tilde{x}_{i,j}|^2 + |\tilde{x}_{i,j+1} - \tilde{x}_{i,j}|^2 }
\end{align}
for a two-dimensional signal $\tilde{\x}$. The factor 0.1 is used to achieve good performance for $\beta=\sigma_e^2$ in case of denoising ($\A=\I_n$). Obviously, it does not affect the comparison between the methods, since $s(\tilde{\x})$ is multiplied by $\beta$ that can be set arbitrarily.
Note that $s(\tilde{\x})$ is convex, and thus $f_{LS}(\tilde{\x})$ and $f_{BP}(\tilde{\x})$ are also convex functions.
We choose to minimize them by the same method: 100 iterations of FISTA \cite{beck2009fast}, which is basically a variant of ISTA (see \eqref{Eq_ista} in the appendix) that is incorporated with Nesterov's accelerated gradient \cite{nesterov1983method}. 
The step size $\mu$ is the typical 1 over the Lipschitz constant of $\nabla  \ell(\tilde{\x})$, 
which in our case 
can be computed as 1 over the spectral norm of the constant Hessian matrix $\nabla^2 \ell$, i.e. $\mu = 1 / \| \P_A \| = 1$ for BP recovery and $\mu = 1 / \| \A^T\A \|$ (computed by the power method) for LS recovery.
This common choice of step size is known to ensure convergence in the convex setting \cite{beck2009fast}. 
Several methods for performing proximal mapping of $s(\tilde{\x})$ (i.e. Gaussian denoising associated with the TV prior) exist  \cite{goldstein2009split, beck2009fast2}. Here, we choose to apply split Bregman method \cite{goldstein2009split}.  
The experiments are performed on the following eight classical test images: {\em cameraman}, {\em house}, {\em peppers}, {\em Lena}, {\em Barbara}, {\em boat}, {\em hill} and {\em couple}.

\subsubsection{Super-resolution}
\label{sec_exp_tv_sr}

We compare the performance of $\hat{\x}_{LS}$ and $\hat{\x}_{BP}$ for SR with Gaussian anti-aliasing kernel (defined in Section \ref{sec_math_ver_sr}) and scale factor of 3. We consider the noiseless case $\sigma_e=0$, as well as the case of Gaussian noise with $\sigma_e=\sqrt{2}$.
For both estimators we initialize FISTA with the bicubic upsampling of $\y$.
For BP, the operator $\A^\dagger$ has fast implementation using the conjugate gradient method \cite{hestenes1952methods}. 
Fig. \ref{fig:SR_tv} shows the PSNR of the reconstructions, averaged over all images, for different values of the regularization parameter $\beta$.
Fig. \ref{fig:SR_tv_psnr_vs_iter} shows the average PSNR as a function of the iteration number, where for each estimator we use the value of $\beta$ which has led to its best results in Fig. \ref{fig:SR_tv} (0.25 for LS and 16 for BP in Fig. \ref{fig:SR_tv_noiseless2}; 0.5 for LS and 46 for BP in Fig. \ref{fig:SR_tv_noisy2}). It can be seen that $\hat{\x}_{BP}$ converges somewhat faster than $\hat{\x}_{LS}$. 
In Figs. \ref{fig:SR_cameraman_ls_tv} and \ref{fig:SR_cameraman_bp_tv} we also display the results for {\em cameraman} image in the noiseless case.

Note the agreement of the obtained results with the observations from Section \ref{sec_math}, even though they have been established for a much simpler convex prior. 
In the noiseless case, $\hat{\x}_{BP}$ outperforms $\hat{\x}_{LS}$ for any value of $\beta$,
while in the noisy scenario, this does not hold. However, even in the latter case, $\hat{\x}_{BP}$ (with good tuning of $\beta$) outperforms $\hat{\x}_{LS}$ (with good tuning of $\beta$). Yet, the gap between them (for optimal tuning) is smaller than in the noiseless case.

\begin{figure}[t]
  \centering
  \begin{subfigure}[b]{0.5\linewidth}
    \centering\includegraphics[width=120pt]{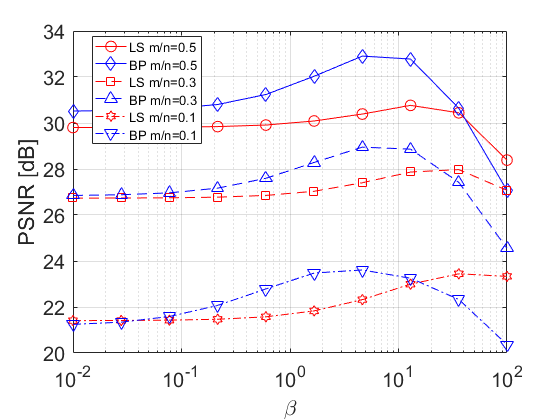}
    \caption{\label{fig:CS_tv_noiseless}}
  \end{subfigure}%
  \begin{subfigure}[b]{0.5\linewidth}
    \centering\includegraphics[width=120pt]{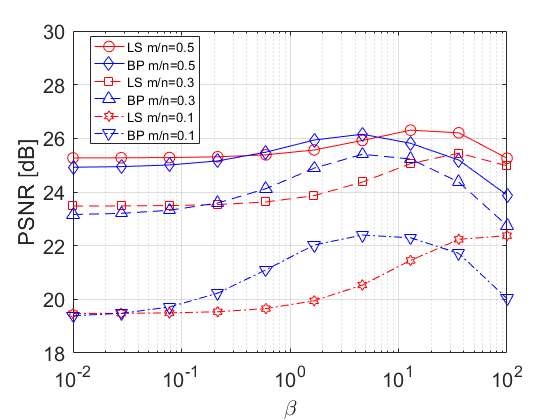}
    \caption{\label{fig:CS_tv_noisy}}
  \end{subfigure}
  \caption{Compressed sensing with Gaussian measurements, using TV prior and 500 iterations of FISTA. PSNR (averaged over 8 test images) vs. $\beta$ (regularization parameter), for (\subref{fig:CS_tv_noiseless}) $\sigma_e=0$, and (\subref{fig:CS_tv_noisy}) SNR of 20 dB.}
\label{fig:CS_tv}

\vspace{2mm}

  \centering
  \begin{subfigure}[b]{0.5\linewidth}
    \centering\includegraphics[width=120pt]{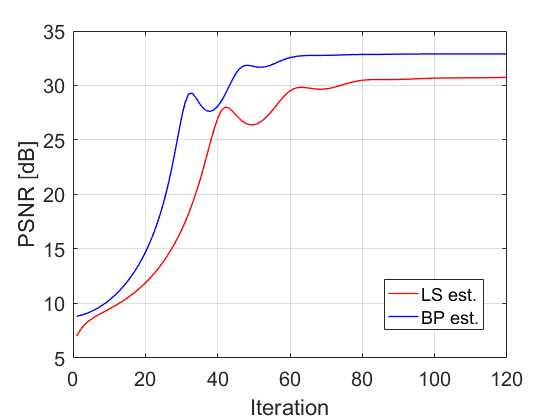}
    \caption{\label{fig:CS_tv_noiseless2}}
  \end{subfigure}%
  \begin{subfigure}[b]{0.5\linewidth}
    \centering\includegraphics[width=120pt]{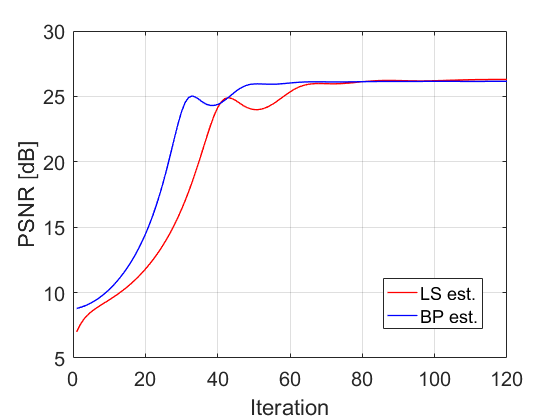}
    \caption{\label{fig:CS_tv_noisy2}}
  \end{subfigure}
  \caption{Compressed sensing with $m=0.5n$ Gaussian measurements, using TV prior. PSNR (for best uniform setting of $\beta$, averaged over 8 test images) vs. FISTA iteration number, for (\subref{fig:CS_tv_noiseless2}) $\sigma_e=0$, and (\subref{fig:CS_tv_noisy2}) SNR of 20 dB.}
\label{fig:CS_tv_psnr_vs_iter}
\end{figure}

\begin{figure}[t]
  \centering
  \begin{subfigure}[b]{0.49\linewidth}
    \centering\includegraphics[width=100pt]{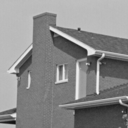}
    \caption{Original image}
  \end{subfigure}%
  \begin{subfigure}[b]{0.49\linewidth}
    \centering\includegraphics[width=100pt]{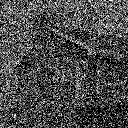}
    \caption{$\A^\dagger \y$ (7.93 dB)}
  \end{subfigure}%
\\
\vspace{1mm}
  \begin{subfigure}[b]{0.49\linewidth}
    \centering\includegraphics[width=100pt]{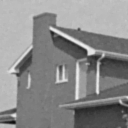}
    \caption{LS-TV (34.12 dB)}
  \end{subfigure}
  \begin{subfigure}[b]{0.49\linewidth}
    \centering\includegraphics[width=100pt]{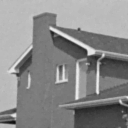}
    \caption{BP-TV (37.19 dB)}
  \end{subfigure}
  \caption{Compressed sensing with $m=0.5n$ Gaussian measurements and $\sigma_e=0$ for {\em house} image. From left to right and from top to bottom: original image, naive $\A^\dagger \y$, reconstruction of LS fidelity with TV prior, and reconstruction of BP fidelity with TV prior.}
\label{fig:CS_noiseless_TV}
\end{figure}

\subsubsection{Deblurring}
\label{sec_exp_tv_deb}

We compare the two estimators for the widely examined $9 \times 9$ uniform blur kernel mentioned in Section \ref{sec_math_ver_deb}. We make the common assumption of circular shift-invariant blur operator, which allows very fast implementation of the gradient steps in the optimization of both cost functions using Fast Fourier Transform (FFT).
We consider two levels of Gaussian noise: $\sigma_e=\sqrt{0.3}$ and $\sigma_e=\sqrt{2}$.
For both estimators we initialize FISTA with $\y$, and for $\hat{\x}_{BP}$ we use $\epsilon = 0.01\sigma_e^2$. 
Fig. \ref{fig:deb_tv} shows the average PSNR for different values of $\beta$, and Fig. \ref{fig:deb_tv_psnr_vs_iter} shows the average PSNR as a function of the iteration number, where each estimator uses the best $\beta$ from Fig. \ref{fig:deb_tv} (0.3 for LS and 20.5 for BP in Fig. \ref{fig:deb_tv_noiseless2}; 0.98 for LS and 19.5 for BP in Fig. \ref{fig:deb_tv_noisy2}).
Note that $\hat{\x}_{BP}$ converges much faster than $\hat{\x}_{LS}$. The difference here for deblurring is more significant than for SR.
Visual results for {\em couple} image in the case of $\sigma_e=\sqrt{2}$ are presented in Figs. \ref{fig:deb_couple_ls_tv} and \ref{fig:deb_couple_bp_tv}.

The obtained results agree with the observations in Section \ref{sec_math_ver_deb}, in the sense that there exist settings of $(\beta,\epsilon)$ for which $\hat{\x}_{BP}$ outperforms $\hat{\x}_{LS}$. Presumably, even for the more complex TV prior, this is due to a better handling 
\tomtb{of $\A$ whose condition number is very large}. 
As expected, the performance gap between the estimators (for optimal tuning) decreases when the noise level is higher. However, it is still highly in favor of $\hat{\x}_{BP}$.

\subsubsection{Compressed sensing}
\label{sec_exp_tv_cs}

\tomt{
We compare the performance of $\hat{\x}_{LS}$ and $\hat{\x}_{BP}$ for CS with Gaussian measurement matrix (i.e. $A_{ij} \sim \mathcal{N}(0,1/m)$), for which the two cost functions differ (see the discussion in Section \ref{sec_math_ver_cs}). 
In these CS experiments (only) we decrease the size of the test images to 128$\times$128 pixels, as there is no efficient way to implement the operators $\A$ and $\A^T$ for large dimensions. 
We consider compression ratios of $m/n=0.1$, $m/n=0.3$, and $m/n=0.5$. For each of them we examine the noiseless case and the case of Gaussian noise with signal-to-noise ratio (SNR) of 20 dB.
For both estimators we initialize FISTA with zero and use 500 iterations.
As we compute $\A^\dagger$ in advance, both estimators have similar computational cost per iteration. 
Fig.~\ref{fig:CS_tv} shows the average PSNR for different values of $\beta$. For $m/n=0.5$ we show in Fig.~\ref{fig:CS_tv_psnr_vs_iter} the average PSNR vs.~the iteration number, where each estimator uses the best $\beta$ from Fig.~\ref{fig:CS_tv}.
Again, note that $\hat{\x}_{BP}$ requires less iterations than $\hat{\x}_{LS}$. 
Visual results for {\em house} image in the noiseless case are presented in Fig.~\ref{fig:CS_noiseless_TV}.
}

\tomt{
The results show correlation with the observations in Section~\ref{sec_math}.
In the noiseless case, when the $m/n$ ratio increases (and thus 
\tomtb{the condition number of $\A\A^T$ increases}, 
e.g. see Figs.~\ref{fig:CS_eigenvalues} and \ref{fig:CS_eigenvalues2}) the performance gap between BP and LS increases in favor of BP. In the noisy case, when the $m/n$ ratio increases the BP estimator becomes more sensitive to noise (due to the increase in the number of singular values that are smaller than 1, \tomtb{again, see Figs.~\ref{fig:CS_eigenvalues} and \ref{fig:CS_eigenvalues2})}. 
}

\begin{figure}[t]
  \centering
  \begin{subfigure}[b]{0.5\linewidth}
    \centering\includegraphics[width=120pt]{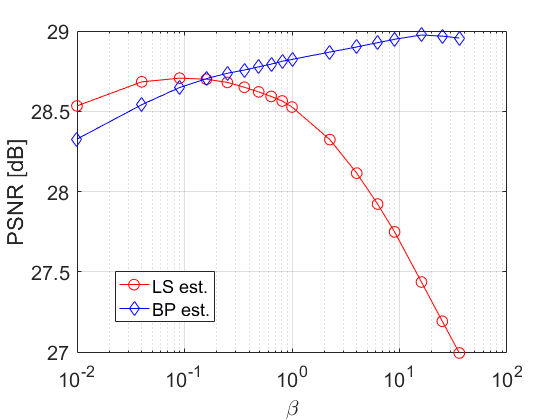}
    \caption{\label{fig:SR_bm3d_noiseless}}
  \end{subfigure}%
  \begin{subfigure}[b]{0.5\linewidth}
    \centering\includegraphics[width=120pt]{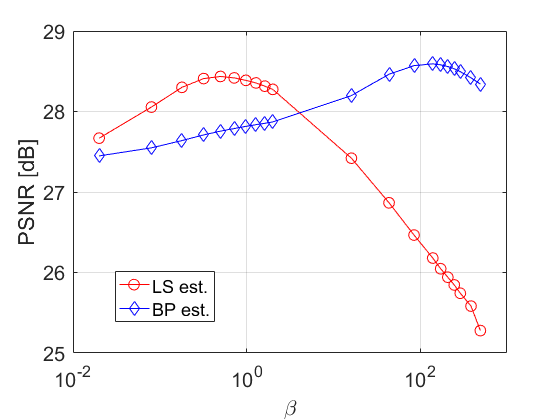}
    \caption{\label{fig:SR_bm3d_noisy}}
  \end{subfigure}
  \caption{Super-resolution with Gaussian filter and scale factor of 3, using BM3D prior and 200 iterations of FISTA. PSNR (averaged over 8 test images) vs. $\beta$ (regularization parameter), for (\subref{fig:SR_bm3d_noiseless}) $\sigma_e=0$, 
and (\subref{fig:SR_bm3d_noisy}) $\sigma_e=\sqrt{2}$.}
\label{fig:SR_bm3d}

\vspace{2mm}

  \centering
  \begin{subfigure}[b]{0.5\linewidth}
    \centering\includegraphics[width=120pt]{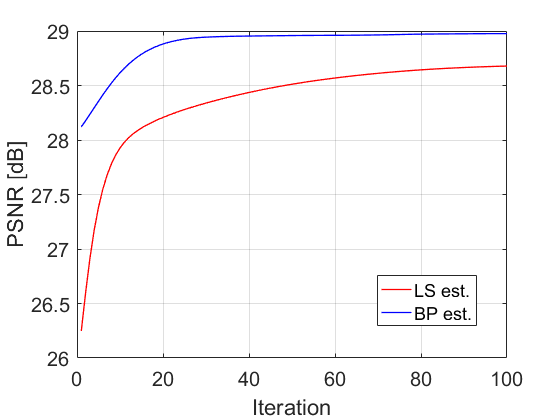}
    \caption{\label{fig:SR_bm3d_noiseless2}}
  \end{subfigure}%
  \begin{subfigure}[b]{0.5\linewidth}
    \centering\includegraphics[width=120pt]{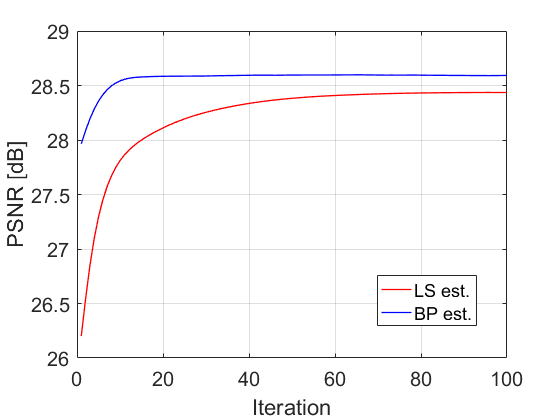}
    \caption{\label{fig:SR_bm3d_noisy2}}
  \end{subfigure}
  \caption{Super-resolution with Gaussian filter and scale factor of 3, using BM3D prior. PSNR (for best uniform setting of $\beta$, averaged over 8 test images) vs. FISTA iteration number, for (\subref{fig:SR_bm3d_noiseless2}) $\sigma_e=0$, and (\subref{fig:SR_bm3d_noisy2}) $\sigma_e=\sqrt{2}$.}
\label{fig:SR_bm3d_psnr_vs_iter}
\end{figure}

\begin{figure}[t]
  \centering
  \begin{subfigure}[b]{0.5\linewidth}
    \centering\includegraphics[width=120pt]{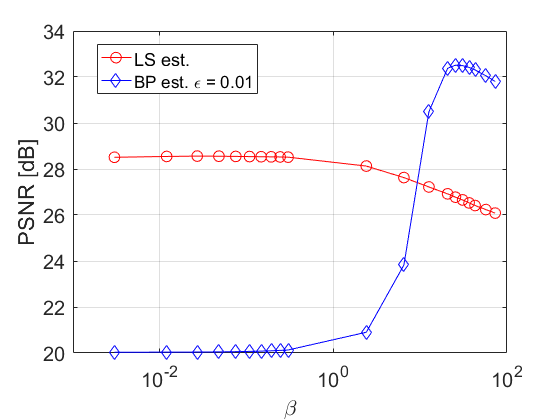}
    \caption{\label{fig:deb_bm3d_noiseless}}
  \end{subfigure}%
  \begin{subfigure}[b]{0.5\linewidth}
    \centering\includegraphics[width=120pt]{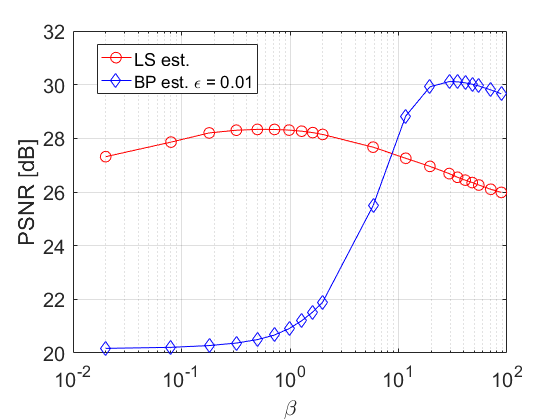}
    \caption{\label{fig:deb_bm3d_noisy}}
  \end{subfigure}
  \caption{Deblurring with uniform $9 \times 9$ blur kernel, using BM3D prior and 200 iterations of FISTA. PSNR (averaged over 8 test images) vs. $\beta$ (regularization parameter), for (\subref{fig:deb_bm3d_noiseless}) $\sigma_e=\sqrt{0.3}$, and (\subref{fig:deb_bm3d_noisy}) $\sigma_e=\sqrt{2}$.}
\label{fig:deb_bm3d}

\vspace{5.5mm}

  \centering
  \begin{subfigure}[b]{0.5\linewidth}
    \centering\includegraphics[width=120pt]{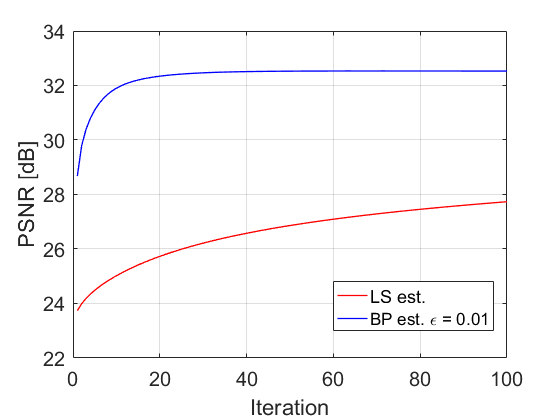}
    \caption{\label{fig:deb_bm3d_noiseless2}}
  \end{subfigure}%
  \begin{subfigure}[b]{0.5\linewidth}
    \centering\includegraphics[width=120pt]{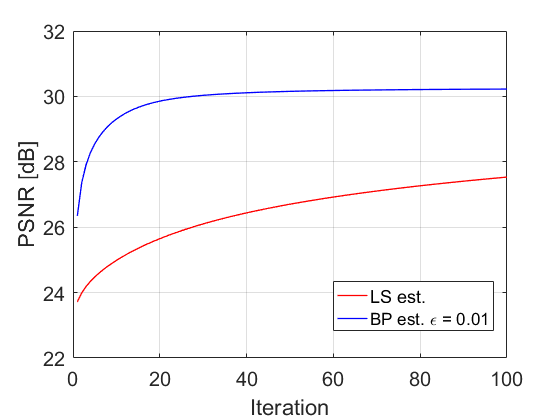}
    \caption{\label{fig:deb_bm3d_noisy2}}
  \end{subfigure}
  \caption{Deblurring with uniform $9 \times 9$ blur kernel, using BM3D prior. PSNR (for best uniform setting of $\beta$, averaged over 8 test images) vs. FISTA iteration number, for (\subref{fig:deb_bm3d_noiseless2}) $\sigma_e=\sqrt{0.3}$, and (\subref{fig:deb_bm3d_noisy2}) $\sigma_e=\sqrt{2}$.}
\label{fig:deb_bm3d_psnr_vs_iter}
\end{figure}

\subsection{BM3D prior}
\label{sec_exp_bm3d}

We turn to compare the performance of the two cost functions for the BM3D prior \cite{dabov2007image}, which is based on 
sparsifying 
a three-dimensional transformation applied to groups of nearest-neighbor (i.e. similar) patches. 
This prior is non-convex. In fact, it is also not clear how to precisely formulate its associated $s(\tilde{\x})$. Yet, when implementing proximal algorithms the proximal mapping of $s(\tilde{\x})$ can be replaced with applying the BM3D denoiser as a ''black-box''.
We use 200 iterations of FISTA to minimize the cost functions 
with typical step sizes as explained above, and the same eight classical test images.

\begin{figure}
  \centering
  \begin{subfigure}[b]{0.5\linewidth}
    \centering\includegraphics[width=120pt]{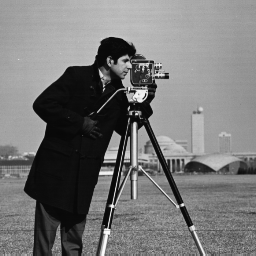}\\
\vspace{1mm}
    \centering\includegraphics[width=58pt]{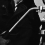}
    \centering\includegraphics[width=58pt]{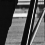}
    \caption{Original image\label{fig:SR_cameraman_x0}}
\vspace{1mm}
  \end{subfigure}%
  \begin{subfigure}[b]{0.5\linewidth}
    \centering\includegraphics[width=120pt]{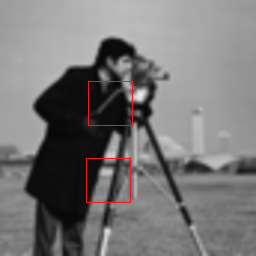}\\
\vspace{1mm}
    \centering\includegraphics[width=58pt]{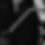}
    \centering\includegraphics[width=58pt]{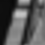}
    \caption{Bicubic (22.83 dB)}
\vspace{1mm}
  \end{subfigure}
  \begin{subfigure}[b]{0.5\linewidth}
    \centering\includegraphics[width=120pt]{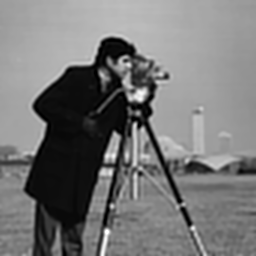}\\
\vspace{1mm}
    \centering\includegraphics[width=58pt]{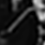}
    \centering\includegraphics[width=58pt]{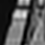}
    \caption{LS-TV (24.37 dB)\label{fig:SR_cameraman_ls_tv}}
\vspace{1mm}
  \end{subfigure}%
  \begin{subfigure}[b]{0.5\linewidth}
    \centering\includegraphics[width=120pt]{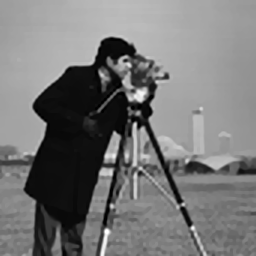}\\
\vspace{1mm}
    \centering\includegraphics[width=58pt]{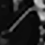}
    \centering\includegraphics[width=58pt]{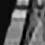}
    \caption{BP-TV (24.94 dB)\label{fig:SR_cameraman_bp_tv}}
\vspace{1mm}
  \end{subfigure}
  \begin{subfigure}[b]{0.5\linewidth}
    \centering\includegraphics[width=120pt]{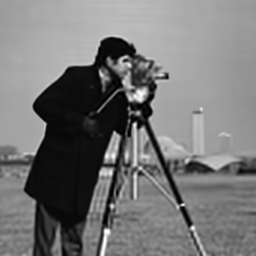}\\
\vspace{1mm}
    \centering\includegraphics[width=58pt]{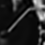}
    \centering\includegraphics[width=58pt]{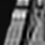}
    \caption{LS-BM3D (25.02 dB)\label{fig:SR_cameraman_ls_bm3d}}
  \end{subfigure}%
  \begin{subfigure}[b]{0.5\linewidth}
    \centering\includegraphics[width=120pt]{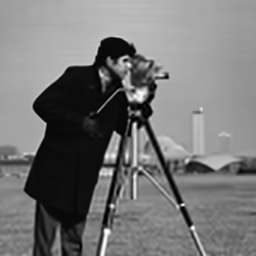}\\
\vspace{1mm}
    \centering\includegraphics[width=58pt]{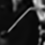}
    \centering\includegraphics[width=58pt]{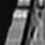}
    \caption{BP-BM3D (25.38 dB)\label{fig:SR_cameraman_bp_bm3d}}
  \end{subfigure}
  \caption{Super-resolution with Gaussian filter, scale factor of 3 and $\sigma_e=0$ for {\em cameraman} image. From left to right and from top to bottom: original image, bicubic upsampling, reconstruction of LS fidelity with TV prior, reconstruction of BP fidelity with TV prior, reconstruction of LS fidelity with BM3D prior, and reconstruction of BP fidelity with BM3D prior.}
\label{fig:SR_noiseless_cameraman}
\end{figure}

\begin{figure}
  \centering
  \begin{subfigure}[b]{0.5\linewidth}
    \centering\includegraphics[width=120pt]{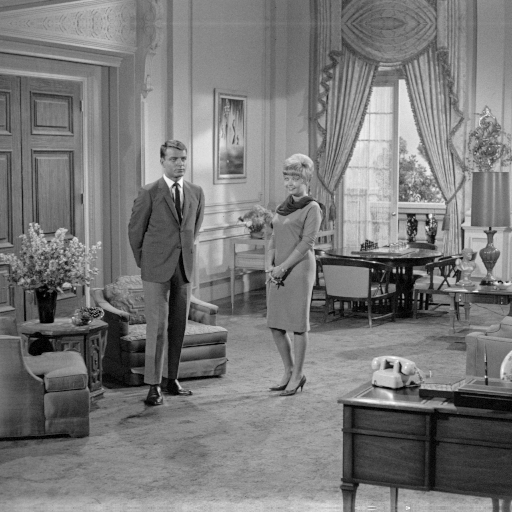}\\
\vspace{1mm}
    \centering\includegraphics[width=58pt]{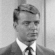}
    \centering\includegraphics[width=58pt]{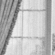}
    \caption{Original image\label{fig:deb_couple_x0}}
\vspace{1mm}
  \end{subfigure}%
  \begin{subfigure}[b]{0.5\linewidth}
    \centering\includegraphics[width=120pt]{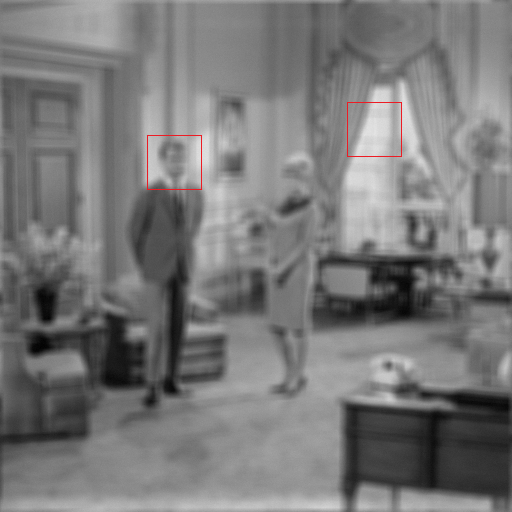}\\
\vspace{1mm}
    \centering\includegraphics[width=58pt]{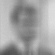}
    \centering\includegraphics[width=58pt]{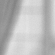}
    \caption{Blurred and noisy image}
\vspace{1mm}
  \end{subfigure}
  \begin{subfigure}[b]{0.5\linewidth}
    \centering\includegraphics[width=120pt]{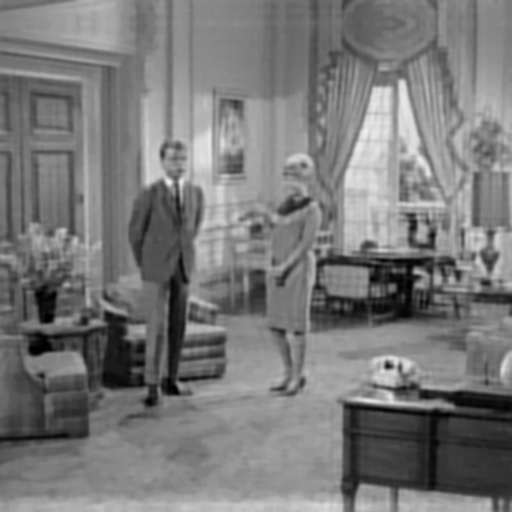}\\
\vspace{1mm}
    \centering\includegraphics[width=58pt]{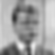}
    \centering\includegraphics[width=58pt]{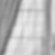}
    \caption{LS-TV (26.45 dB)\label{fig:deb_couple_ls_tv}}
\vspace{1mm}
  \end{subfigure}%
  \begin{subfigure}[b]{0.5\linewidth}
    \centering\includegraphics[width=120pt]{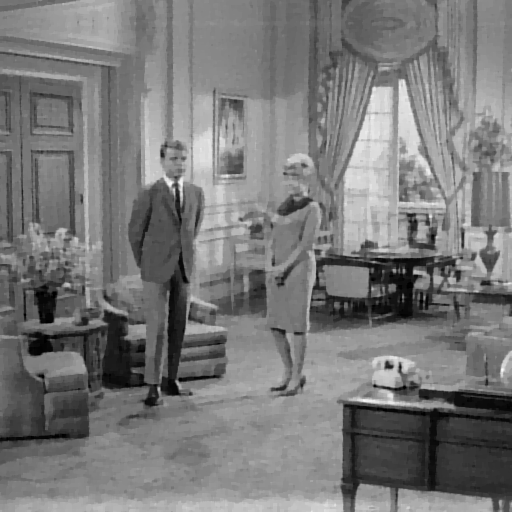}\\
\vspace{1mm}
    \centering\includegraphics[width=58pt]{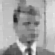}
    \centering\includegraphics[width=58pt]{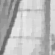}
    \caption{BP-TV (29.28 dB)\label{fig:deb_couple_bp_tv}}
\vspace{1mm}
  \end{subfigure}
  \begin{subfigure}[b]{0.5\linewidth}
    \centering\includegraphics[width=120pt]{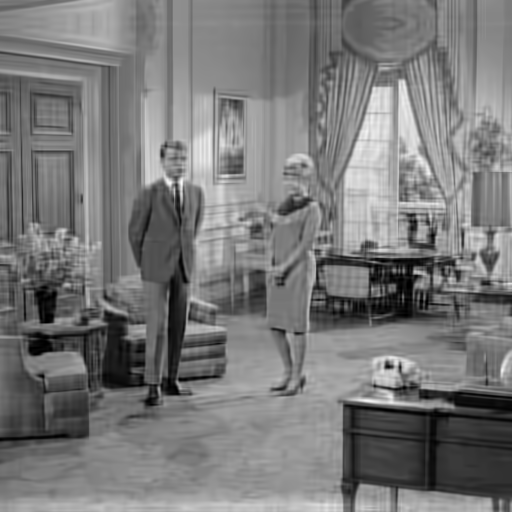}\\
\vspace{1mm}
    \centering\includegraphics[width=58pt]{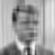}
    \centering\includegraphics[width=58pt]{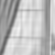}
    \caption{LS-BM3D (28.74 dB)\label{fig:deb_couple_ls_bm3d}}
  \end{subfigure}%
  \begin{subfigure}[b]{0.5\linewidth}
    \centering\includegraphics[width=120pt]{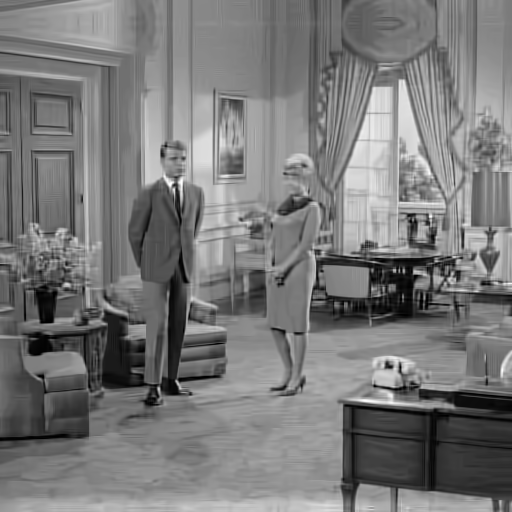}\\
\vspace{1mm}
    \centering\includegraphics[width=58pt]{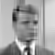}
    \centering\includegraphics[width=58pt]{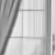}
    \caption{BP-BM3D (30.38 dB)\label{fig:deb_couple_bp_bm3d}}
  \end{subfigure}
  \caption{Deblurring with uniform $9 \times 9$ blur kernel and $\sigma_e=\sqrt{2}$ for {\em couple} image. From left to right and from top to bottom: original image, blurred and noisy image, reconstruction of LS fidelity with TV prior, reconstruction of BP fidelity with TV prior, reconstruction of LS fidelity with BM3D prior, and reconstruction of BP fidelity with BM3D prior.}
\label{fig:deb_noisy_couple}
\end{figure}

\subsubsection{Super-resolution}
\label{sec_exp_bm3d_sr}

We repeat the two SR experiments of Section \ref{sec_exp_tv_sr}.
Fig. \ref{fig:SR_bm3d} shows the average PSNR for different values of $\beta$, and
Fig. \ref{fig:SR_bm3d_psnr_vs_iter} shows the average PSNR as a function of the iteration number, where each estimator uses the best $\beta$ from Fig. \ref{fig:SR_bm3d} (0.09 for LS and 16 for BP in Fig. \ref{fig:SR_bm3d_noiseless2}; 0.5 for LS and 140 for BP in Fig. \ref{fig:SR_bm3d_noisy2}). Again, note that $\hat{\x}_{BP}$ converges much faster than $\hat{\x}_{LS}$. 
In Figs. \ref{fig:SR_cameraman_ls_bm3d} and \ref{fig:SR_cameraman_bp_bm3d} we display the results for {\em cameraman} image in the noiseless case.

Note the strong correlation between the obtained results and the observations from Section \ref{sec_math}, even though the prior is highly non-convex.
In the noiseless case, $\hat{\x}_{BP}$ outperforms $\hat{\x}_{LS}$ for a large range of $\beta$. For very small values of $\beta$ it is inferior to $\hat{\x}_{LS}$, but with only a small gap. 
From a practitioner point of view, the advantages of using the BP cost here are still clear, since when $\beta$ is well-tuned (for each of the cost functions) $\hat{\x}_{BP}$ is significantly better.
Note that in the examined noisy scenario, well-tuned $\hat{\x}_{BP}$ is still better than well-tuned $\hat{\x}_{LS}$, but the gap decreases.

\subsubsection{Deblurring}
\label{sec_exp_bm3d_deb}

We repeat the two deblurring experiments of Section \ref{sec_exp_tv_deb}.
Fig. \ref{fig:deb_bm3d} shows the average PSNR for different values of $\beta$, and Fig. \ref{fig:deb_bm3d_psnr_vs_iter} shows the average PSNR as a function of the iteration number, where each estimator uses the best $\beta$ from Fig. \ref{fig:deb_bm3d} (0.027 for LS and 25.5 for BP in Fig. \ref{fig:deb_bm3d_noiseless2}; 0.5 for LS and 29.5 for BP in Fig. \ref{fig:deb_bm3d_noisy2}).
Figs. \ref{fig:deb_couple_ls_bm3d} and \ref{fig:deb_couple_bp_bm3d} present visual results for {\em couple} image in the case of $\sigma_e=\sqrt{2}$.
The observations that have been made for TV prior stay the same here for the BM3D prior: There exist settings of $(\beta,\epsilon)$ for which $\hat{\x}_{BP}$ significantly outperforms $\hat{\x}_{LS}$ and converges faster.

\subsection{DCGAN prior}

The developments in deep learning \cite{goodfellow2016deep} in the recent years have led to significant improvement in learning generative models. 
Methods like variational auto-encoders (VAEs) \cite{kingma2013auto} and generative adversarial networks (GANs) \cite{goodfellow2014generative} have found success at modeling data distributions. 
This has naturally led to using pre-trained generative models as priors in imaging inverse problems \cite{bora2017compressed}.
Since in popular generative models \cite{kingma2013auto, goodfellow2014generative} a generator $\mathcal{G}(\cdot)$ learns a mapping from a low dimensional space $\z \in \mathbb{R}^d$ to the signal space $\mathcal{G}(\z) \subset \mathbb{R}^n$, one can search for a reconstruction of $\x$ only in the range of the generator.
This can be formulated by the following non-convex prior
\begin{align}
\label{Eq_gan_prior}
s(\tilde{\x})=
\begin{cases} 
      0, & \exists \tilde{\z} \in \Bbb R^d: \tilde{\x}=\mathcal{G}(\tilde{\z}) \\
      +\infty,  & otherwise
   \end{cases}.
\end{align}
Plugging \eqref{Eq_gan_prior} into the typical cost function \eqref{Eq_cost_typical}, we get the objective
\begin{align}
\label{Eq_cost_gan_typical}
f_{LS}(\tilde{\z}) = \| \y-\A \mathcal{G}(\tilde{\z}) \|_2^2.
\end{align}
Note that for this prior, a regularization parameter $\beta$ is not required. The recovery of the latent image $\x$ is given by $\hat{\x}_{LS}=\mathcal{G}(\hat{\z}_{LS})$, where $\hat{\z}_{LS}$ is a minimizer of \eqref{Eq_cost_gan_typical}, which can be obtained by backpropagation and standard gradient based optimizers. 

The technique above has been examined recently in \cite{bora2017compressed}.
Here, we compare it with the one obtained by a similar approach that uses the BP cost function \eqref{Eq_cost_bp}, i.e. we plug \eqref{Eq_gan_prior} into \eqref{Eq_cost_bp}, to get the objective
\begin{align}
\label{Eq_cost_gan_bp}
f_{BP}(\tilde{\z}) = \| \A^\dagger (\y-\A \mathcal{G}(\tilde{\z})) \|_2^2,
\end{align}
and recover $\x$ by $\hat{\x}_{BP}=\mathcal{G}(\hat{\z}_{BP})$, where $\hat{\z}_{BP}$ is a minimizer of \eqref{Eq_cost_gan_bp}.

We use the CelebA dataset \cite{liu2015deep} and Tensorflow package \cite{abadi2016tensorflow} to train a generator using DCGAN architecture \cite{radford2015unsupervised} on the cropped version of the images (64$\times$64 pixels), as done in \cite{bora2017compressed}. We use the first 200,000 images (out of 202,599) for training, and the training procedure follows the one in \cite{radford2015unsupervised, bora2017compressed}. 
At test time, all 
the optimizations with respect to $\z$ are performed using: ADAM \cite{kingma2014adam} with learning rate of 0.1 (as done in \cite{bora2017compressed}), 
same 10 random initializations of $\tilde{\z}$, and 2000 iterations, which suffice for ensuring that the objectives \eqref{Eq_cost_gan_typical} and \eqref{Eq_cost_gan_bp} stop decreasing. The value of $\tilde{\z}$ that gives the lowest objective is chosen.

\subsubsection{Super-resolution}

We compare the performance of $\hat{\x}_{LS}$ and $\hat{\x}_{BP}$ for SR with Gaussian anti-aliasing kernel (defined in Section \ref{sec_math_ver_sr}) and scale factor of 3. 
Table \ref{table:dcgan} shows the PSNR results for the different cost functions, averaged over the last \tomt{50 images} in CelebA (these images are not included in the training data). Several visual results are shown in Fig. \ref{fig:SR_dcgan}.

It can be seen that the BP fidelity yields higher average PSNR and perceptually better recoveries.
In fact, in each of the \tomt{50 examined images} $\hat{\x}_{BP}$ has obtained higher PSNR than $\hat{\x}_{LS}$.
This behavior agrees with the previous experiments that demonstrate the advantages of the BP cost for the noiseless SR problem.
We also note that even though the results of the simple bicubic upsampling are always perceptually worse than the recoveries that use DCGAN, its PSNR is sometimes higher. 
This drawback of GAN-based priors is due to the limited representation capabilities of the generators (sometimes referred to as ''mode collapse'').  A very recent work has suggested to mitigate this deficiency by image-adaptation and back-projections \cite{shady2019image}.

\subsubsection{Compressed sensing}

Due to the small image dimensions, we are able to compare the performance of $\hat{\x}_{BP}$ and $\hat{\x}_{LS}$ for CS with Gaussian measurement matrix (i.e. $A_{ij} \sim \mathcal{N}(0,1/m)$), for which the two cost functions differ (see the discussion in Section \ref{sec_math_ver_cs}).
We use compression ratios of $m/n=0.1$, \tomt{$m/n=0.3$}, and $m/n=0.5$. 
Table \ref{table:dcgan2} shows the PSNR results for the different cost functions, averaged over the last \tomt{50 images} in CelebA. Several visual results are shown in Figs. \ref{fig:CS_dcgan} and \ref{fig:CS_dcgan2}.

\tomt{
The performance gap between $\hat{\x}_{BP}$ and $\hat{\x}_{LS}$ is negligible for $m/n=0.1$, and increases in favor of BP when the $m/n$ ratio increases. 
This behavior 
correlates with the analysis in Section \ref{sec_math} \tomtb{(specifically with \tomtb{Observation~\ref{observ3}})}, which explains such behavior for $\ell_2$ priors by the fact that 
\tomtb{when the $m/n$ ratio increases the condition number of $\A\A^T$ increases as well}.
}

\begin{table}
\small
\renewcommand{\arraystretch}{1.3}
\caption{Reconstruction PSNR [dB] (averaged over 50 images from CelebA) for super-resolution with Gaussian filter and scale factor of 3, using DCGAN prior and ADAM optimizer.} \label{table:dcgan}
\centering
    \begin{tabular}{ | l | l | l | l |}
    \hline
            & Bicubic  & LS est. & BP est. \\ \hline
  SR x3  & 23.04  & 23.02 & 23.77 \\ \hline
    \end{tabular}
\end{table}

\begin{figure}[t]
 \centering
  \subcaptionbox*{Original  \label{fig:x0}}{%
  \includegraphics[width=0.2\columnwidth]{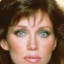}}
  \subcaptionbox*{Bicubic (23.74)  \label{fig:x0}}{%
  \includegraphics[width=0.2\columnwidth]{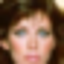}}
  \subcaptionbox*{LS est. (23.96)  \label{fig:x0}}{%
  \includegraphics[width=0.2\columnwidth]{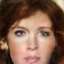}}
  \subcaptionbox*{BP est. (24.47) \label{fig:x0}}{%
  \includegraphics[width=0.2\columnwidth]{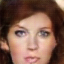}}\\
  \subcaptionbox*{Original  \label{fig:x0}}{%
  \includegraphics[width=0.2\columnwidth]{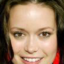}}
  \subcaptionbox*{Bicubic (23.22) \label{fig:x0}}{%
  \includegraphics[width=0.2\columnwidth]{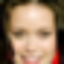}}
  \subcaptionbox*{LS est. (22.88) \label{fig:x0}}{%
  \includegraphics[width=0.2\columnwidth]{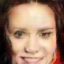}}
  \subcaptionbox*{BP est. (23.23) \label{fig:x0}}{%
  \includegraphics[width=0.2\columnwidth]{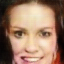}}\\
  \subcaptionbox*{Original  \label{fig:x0}}{%
  \includegraphics[width=0.2\columnwidth]{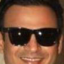}}
  \subcaptionbox*{Bicubic (21.50) \label{fig:x0}}{%
  \includegraphics[width=0.2\columnwidth]{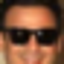}}
  \subcaptionbox*{LS est. (23.15) \label{fig:x0}}{%
  \includegraphics[width=0.2\columnwidth]{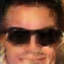}}
  \subcaptionbox*{BP est. (24.11) \label{fig:x0}}{%
  \includegraphics[width=0.2\columnwidth]{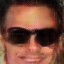}}
 \caption{Super-resolution with Gaussian filter and scale factor of 3, using DCGAN prior.
}\label{fig:SR_dcgan}
\end{figure}

\begin{table}
\small
\renewcommand{\arraystretch}{1.3}
\caption{Reconstruction PSNR [dB] (averaged over 50 images from CelebA) for compressed sensing with Gaussian measurement matrix, using DCGAN prior and ADAM optimizer.} \label{table:dcgan2}
\centering
    \begin{tabular}{ | l | l | l | l |}
    \hline
            & Naive $\A^\dagger\y$  & LS est. & BP est. \\ \hline
  CS $m/n=0.1$  & 12.07  & 22.78 & 22.80 \\ \hline
  CS $m/n=0.3$  & 13.22  & 23.55 & 23.62 \\ \hline
  CS $m/n=0.5$  & 14.71  & 23.67 & 23.82 \\ \hline
    \end{tabular}
\end{table}

%
%

\begin{figure}
 \centering
  \subcaptionbox*{Original  \label{fig:x0}}{%
  \includegraphics[width=0.2\columnwidth]{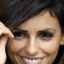}}
  \subcaptionbox*{$\A^\dagger\y$ (11.22)  \label{fig:x0}}{%
  \includegraphics[width=0.2\columnwidth]{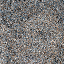}}
  \subcaptionbox*{LS est. (21.15) \label{fig:x0}}{%
  \includegraphics[width=0.2\columnwidth]{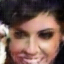}}
  \subcaptionbox*{BP est. (21.00) \label{fig:x0}}{%
  \includegraphics[width=0.2\columnwidth]{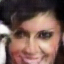}}\\
  \subcaptionbox*{Original  \label{fig:x0}}{%
  \includegraphics[width=0.2\columnwidth]{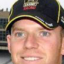}}
  \subcaptionbox*{$\A^\dagger\y$ (11.79) \label{fig:x0}}{%
  \includegraphics[width=0.2\columnwidth]{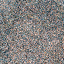}}
  \subcaptionbox*{LS est. (18.51) \label{fig:x0}}{%
  \includegraphics[width=0.2\columnwidth]{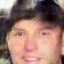}}
  \subcaptionbox*{BP est. (18.36) \label{fig:x0}}{%
  \includegraphics[width=0.2\columnwidth]{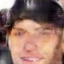}}\\
  \subcaptionbox*{Original  \label{fig:x0}}{%
  \includegraphics[width=0.2\columnwidth]{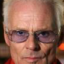}}
  \subcaptionbox*{$\A^\dagger\y$ (13.05) \label{fig:x0}}{%
  \includegraphics[width=0.2\columnwidth]{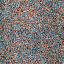}}
  \subcaptionbox*{LS est. (22.11) \label{fig:x0}}{%
  \includegraphics[width=0.2\columnwidth]{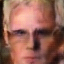}}
  \subcaptionbox*{BP est. (22.10) \label{fig:x0}}{%
  \includegraphics[width=0.2\columnwidth]{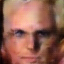}}
\caption{Compressed sensing with $m=0.1n$ Gaussian measurements, using DCGAN prior. 
}\label{fig:CS_dcgan}
\end{figure}


\begin{figure}
 \centering
  \subcaptionbox*{Original  \label{fig:x0}}{%
  \includegraphics[width=0.2\columnwidth]{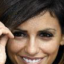}}
  \subcaptionbox*{$\A^\dagger\y$ (11.73)  \label{fig:x0}}{%
  \includegraphics[width=0.2\columnwidth]{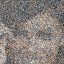}}
  \subcaptionbox*{LS est. (21.89) \label{fig:x0}}{%
  \includegraphics[width=0.2\columnwidth]{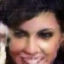}}
  \subcaptionbox*{BP est. (21.80) \label{fig:x0}}{%
  \includegraphics[width=0.2\columnwidth]{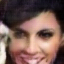}}\\
  \subcaptionbox*{Original  \label{fig:x0}}{%
  \includegraphics[width=0.2\columnwidth]{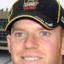}}
  \subcaptionbox*{$\A^\dagger\y$ (12.25) \label{fig:x0}}{%
  \includegraphics[width=0.2\columnwidth]{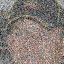}}
  \subcaptionbox*{LS est. (18.77) \label{fig:x0}}{%
  \includegraphics[width=0.2\columnwidth]{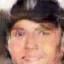}}
  \subcaptionbox*{BP est. (19.00) \label{fig:x0}}{%
  \includegraphics[width=0.2\columnwidth]{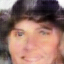}}\\
  \subcaptionbox*{Original  \label{fig:x0}}{%
  \includegraphics[width=0.2\columnwidth]{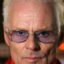}}
  \subcaptionbox*{$\A^\dagger\y$ (13.59) \label{fig:x0}}{%
  \includegraphics[width=0.2\columnwidth]{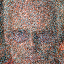}}
  \subcaptionbox*{LS est. (22.85) \label{fig:x0}}{%
  \includegraphics[width=0.2\columnwidth]{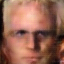}}
  \subcaptionbox*{BP est. (22.97) \label{fig:x0}}{%
  \includegraphics[width=0.2\columnwidth]{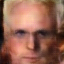}}
\caption{Compressed sensing with $m=0.5n$ Gaussian measurements, using DCGAN prior. 
}\label{fig:CS_dcgan2}
\end{figure}

\section{Conclusion}
\label{sec_conclusion}

In this work we examined the BP fidelity term for ill-posed linear inverse problems. This term has only been used implicitly by the recently proposed iterative denoising and backward projections (IDBP) framework, and is an alternative to the least squares (LS) term, which is the common choice in most works.
We showed that IDBP is essentially a specific optimization scheme, namely the proximal gradient method (known also as ISTA), for minimizing the cost function induced by the BP fidelity term. 
We 
analytically compared the two fidelity terms---BP and LS---for the case of $\ell_2$-type prior functions, and obtained mathematically-backed observations in favor of the BP term 
when the condition number of $\A\A^T$ is large (which is the case
 in many applications, such as super-resolution and deblurring).
\tomt{Furthermore, we showed that it is possible to leverage prior knowledge on $\x$ to increase the coverage of the observations.} 
Finally, we empirically demonstrated that the behavior for sophisticated priors, such as TV, BM3D and DCGAN, strongly correlates with 
the theoretically backed properties that we established for $\ell_2$ priors.
While the mathematical performance analysis in this work is done only for $\ell_2$ priors, it provides a good characterization for the advantages of BP and LS compared to each other. Yet, we believe that there are other factors that should be explored with respect to the new fidelity term, such as its behavior with non-convex priors or its effect on the convergence speed of iterative optimization algorithms.

\appendices

\section{The Connection Between IDBP \cite{tirer2019image} and $f_{BP}(\tilde{\x})$}
\label{sec_idbp_int}

\subsection{Background}
\label{sec_idbp_background}

The iterative denoising and backward projections (IDBP) framework \cite{tirer2019image} is inspired by the plug-and-play priors concept \cite{venkatakrishnan2013plug}, which encourages the usage of existing Gaussian denoisers as ''black boxes'' to implicitly dictate the prior $s(\tilde{\x})$ when solving inverse problems. 
Such an approach allows one to use sophisticated denoising methods even when it is not clear how to formulate their associated priors, e.g. convolutional neural network (CNN) denoisers.

Several plug-and-play works have been published \cite{venkatakrishnan2013plug, sreehari2016plug, romano2017little,  teodoro2016image, kamilov2017plug, chan2017plug, 
 zhang2017learning, sun2019online, ono2017primal}. Most of them 
consider the typical cost function \eqref{Eq_cost_typical}
and directly minimize it using 
existing iterative optimization schemes, such as FISTA \cite{beck2009fast}, ADMM \cite{boyd2011distributed} or quadratic penalty method \cite{nocedal2006sequential}, that include steps in which the proximal mapping of $s(\tilde{\x})$ is used (as explained below, this mapping is equivalent to Gaussian denoising under the prior $s(\tilde{\x})$).

Recently, \cite{tirer2019image} has suggested, after several 
manipulations, to solve a different optimization problem
\begin{align}
\label{Eq_cost_func_our_old}
\minim{\tilde{\x}, \tilde{\z}} \,\,\, \frac{1}{2(\sigma_e+\delta)^2} \| \tilde{\z}-\tilde{\x} \|_2^2 + s(\tilde{\x}) \,\,\,\, \textrm{s.t.} \,\,\,\, \A\tilde{\z}= \y,
\end{align}
where $\sigma_e$ is the noise level and $\delta$ is a design parameter. This work has also proposed an adaptive strategy to set $\delta$, which does not depend on the prior and, contrary to cross-validation, does not require a set of ground truth examples. 
It has been suggested in \cite{tirer2019image} to solve \eqref{Eq_cost_func_our_old} using 
a simple alternating minimization scheme that possesses the plug-and-play property, where the prior term $s(\tilde{\x})$ is handled solely by a Gaussian denoising operation $\mathcal{D}(\cdot;\sigma)$ with noise level $\sigma=\sigma_e+\delta$. 
In this iterative method, 
$\tilde{\z}_k$ is obtained by projecting $\tilde{\x}_{k-1}$ onto $\{ \A \Bbb R^n = \y \}$
\begin{align}
\label{Eq_cost_func_our_z_old}
\tilde{\z}_k &= \argmin{\tilde{\z}} \,\, \| \tilde{\z}-\tilde{\x}_{k-1} \|_2^2  \,\,\,\, \textrm{s.t.} \,\,\,\, \A\tilde{\z}= \y \nonumber  \\
&= \A^\dagger\y + (\I_n - \A^\dagger\A)\tilde{\x}_{k-1} \nonumber  \\
&= \tilde{\x}_{k-1} + \A^\dagger (\y - \A \tilde{\x}_{k-1}).
\end{align}
and $\tilde{\x}_k$ is obtained by
\begin{align}
\label{Eq_cost_func_our_x_old}
\tilde{\x}_k &= \argmin{\tilde{\x}} \,\, \frac{1}{2(\sigma_e+\delta)^2} \| \tilde{\z}_{k}-\tilde{\x} \|_2^2 + s(\tilde{\x}) \nonumber \\
& \triangleq \mathcal{D}(\tilde{\z}_{k};\sigma_e+\delta).
\end{align}
The two repeating operations lends the method its name: Iterative Denoising and Backward Projections (IDBP). After a stopping criterion is met, the last $\tilde{\x}_k$ is taken as the estimate of the latent $\x$.
Note that in many cases the operation $\A^\dagger$ can be performed efficiently (e.g. the matrix inversion can be avoided using the conjugate gradient method \cite{hestenes1952methods}), and thus IDBP is dominated by the complexity of the denoising operation, similarly to other plug-and-play techniques.
Using sophisticated denoisers, such as BM3D and CNNs, this algorithm has achieved excellent results for deblurring \cite{tirer2019image, tirer2018icip} and super-resolution \cite{tirer2018super}.

\subsection{Obtaining IDBP by applying ISTA on $f_{BP}(\tilde{\x})$} 
\label{sec_idbp_ista}

Interestingly, there is another way to develop the exact algorithm, which is different from the way it is developed in \cite{tirer2019image}.
First, note that \eqref{Eq_cost_func_our_old} can be solved directly for $\tilde{\z}$. Similar to \eqref{Eq_cost_func_our_z_old}, we get
\begin{align}
\label{Eq_ztilde_full_min}
\tilde{\z}^* = \A^\dagger \y + (\I_n - \A^\dagger \A ) \tilde{\x}.
\end{align}
Substituting (\ref{Eq_ztilde_full_min}) into (\ref{Eq_cost_func_our_old}), 
we reach $\minim{\tilde{\x}}f_{BP}(\tilde{\x})$ 
with a specific value of the regularization parameter, i.e. $\beta = (\sigma_e+\delta)^2$.
Therefore, IDBP is essentially a specific method to minimize the $f_{BP}(\tilde{\x})$ cost function.
Let us show that this method coincides with applying the proximal gradient method \cite{beck2009fast, combettes2011proximal}, popularized under the name ISTA\footnote{ISTA is the abbreviation of Iterative Shrinkage-Thresholding Algorithm, initially designed for $s(\tilde{\x})=\|\tilde{\x}\|_1$  \cite{daubechies2004iterative}.}, on $f_{BP}(\tilde{\x})$. Let us define the proximal mapping, which was introduced by Moreau \cite{moreau1965proximite} for convex functions. 
\tomt{Here we do not limit this definition to convex functions, though, we emphasize that previous results for proximal mapping of convex functions do not apply to non-convex functions.}
\begin{definition}
\label{def1}
The proximal mapping of a function $s(\cdot)$ at the point $\tilde{\z}$ is defined by
\begin{align}
\label{def_prox}
\mathrm{prox}_{s(\cdot)}(\tilde{\z}) \triangleq \argmin{\tilde{\x}} \,\, \frac{1}{2} \| \tilde{\z} - \tilde{\x} \|_2^2 + s(\tilde{\x}).
\end{align}
\end{definition}
Clearly, given the same $s(\cdot)$, Gaussian denoising and proximal mapping are tightly connected $\mathcal{D}(\tilde{\z};\sigma) = \mathrm{prox}_{\sigma^2 s(\cdot)}(\tilde{\z})$.

Assuming a differentiable fidelity term $\ell(\tilde{\x})$ with a Lipschitz continuous gradient $\nabla \ell(\tilde{\x})$, applying ISTA on \eqref{Eq_cost_func_general} involves iterations of
\begin{align}
\label{Eq_ista}
\tilde{\x}_k =  \mathrm{prox}_{\mu \beta s(\cdot)}(\tilde{\x}_{k-1} - \mu \nabla  \ell(\tilde{\x}_{k-1})),
\end{align}
where $\mu$ is a step-size, 
which ensures convergence for convex $s(\cdot)$ if it is equal to (or smaller than) 1 over the Lipschitz constant of $\nabla  \ell(\tilde{\x})$  \cite{beck2009fast}.

\begin{proposition}
\label{proposition_ista_idbp}
The IDBP algorithm, given in \eqref{Eq_cost_func_our_z_old} and \eqref{Eq_cost_func_our_x_old}, coincides with applying ISTA  \eqref{Eq_ista} on the cost function $f_{BP}(\tilde{\x})$.
\end{proposition}

\begin{proof}

Let us compute $\nabla  \ell_{BP}(\tilde{\x})$. Using the properties $\P_A \triangleq \A^\dagger \A = \P_A^T=\P_A^2$ and $\P_A \A^\dagger = \A^\dagger$, we get
\begin{align}
\label{Eq_fidelity_idbp_grad}
\nabla \ell_{BP}(\tilde{\x}) &= - \P_A ( \A^\dagger \y - \P_A \tilde{\x} ) \nonumber \\
& = - \A^\dagger ( \y - \A \tilde{\x} ).
\end{align}
The Lipschitz constant of $\nabla  \ell_{BP}(\tilde{\x})$ can be computed here as the spectral norm of the constant
Hessian matrix $\nabla^2 \ell_{BP}$. Therefore, $\mu$ can be chosen as
\begin{align}
\label{Eq_ista_mu_idbp}
\mu =  \frac{1}{\| \nabla^2 \ell_{BP}(\tilde{\x}) \|} = \frac{1}{\| \P_A \|} = 1,
\end{align}
where we use the fact that the spectral norm of a non-trivial orthogonal projection is 1.
Now, due to the connection $\mathcal{D}(\tilde{\z};\sigma) = \mathrm{prox}_{\sigma^2 s(\cdot)}(\tilde{\z})$, \eqref{Eq_ista} can be written as 
\begin{align}
\label{Eq_ista_denoiser}
\tilde{\x}_k =  \mathcal{D}(\tilde{\x}_{k-1} - \mu \nabla  \ell(\tilde{\x}_{k-1}); \sqrt{\mu \beta}).
\end{align}
Finally, by plugging \eqref{Eq_fidelity_idbp_grad} and \eqref{Eq_ista_mu_idbp} into \eqref{Eq_ista_denoiser} and setting $\beta = (\sigma_e+\delta)^2$, we get the IDBP scheme, which is presented in \eqref{Eq_cost_func_our_z_old} and \eqref{Eq_cost_func_our_x_old}.

\end{proof}

The connection between IDBP and ISTA, allows IDBP to adopt the theoretical results of the latter. 
Yet, note that the powerful global convergence (obtaining the optimal value of the objective) of ISTA holds only for denoisers that are associated with convex prior functions \cite{beck2009fast}. This limitation is shared also with ADMM-based plug-and-play schemes \cite{sreehari2016plug}.

\bibliographystyle{ieeetr}


\bibliography{paper_ver_for_arXiv}

\end{document}